\documentclass[lettersize,journal]{IEEEtran}
\usepackage{amsmath,amssymb,amsfonts}
\usepackage{algorithmic}
\usepackage[ruled,vlined,linesnumbered]{algorithm2e}
\usepackage{subcaption}
\usepackage{array}
\usepackage[caption=false,font=normalsize,labelfont=sf,textfont=sf]{subfig}
\usepackage{textcomp}
\usepackage{stfloats}
\usepackage{url}
\usepackage{verbatim}
\usepackage{graphicx}
\usepackage{balance}
\usepackage{cite}
\usepackage{xcolor}
\usepackage{comment}

\newtheorem{theorem}{Theorem}

\newtheorem{proposition}[theorem]{Proposition}

\newenvironment{proof}{\textit{Proof.}}{{\leavevmode\nobreak\hfil\penalty50\hskip0em\vadjust{}
		\nobreak\hfil$\Box$\parfillskip=0pt\finalhyphendemerits=0\par}\vspace{1ex} }

\usepackage{multirow}
\usepackage{array}
\usepackage{booktabs}
\newcolumntype{L}[1]{>{\centering\arraybackslash}m{#1}}

\begin{document}

\title{Conflict-Based Lazy Search for Fast Multi-Manipulator Planning}

\author{Dongliang Zheng,~\IEEEmembership{Member,~IEEE,} Zhipeng Wang,~\IEEEmembership{Member,~IEEE,} Siqi Wang, Yuxi Lu,~\IEEEmembership{Member,~IEEE,} \\ Bin He,~\IEEEmembership{Senior Member,~IEEE,} Hesheng Wang,~\IEEEmembership{Senior Member,~IEEE,} and Panagiotis Tsiotras,~\IEEEmembership{Fellow,~IEEE}% <-this % stops a space
\thanks{This work was supported in part by the Science and Technology Commission of Shanghai Municipality (STCSM) under Grant 25ZR1404010, in part by the Shanghai Pujiang Program under Grant 25PJA140, in part by the National Natural Science Foundation of China under Grant 62088101, in part by the STCSM under Grant 2021SHZDZX0100, in part by the Fundamental Research Funds for the Central Universities.}% <-this % stops a space.
\thanks{Dongliang Zheng, Zhipeng Wang, Siqi Wang, Yuxi Lu, and Bin He are with the Shanghai Research Institute for Intelligent Autonomous Systems, Shanghai Institute of Intelligent Science and Technology, the State Key Laboratory of Autonomous Intelligent Unmanned Systems, and Frontiers Science Center for Intelligent Autonomous Systems of Ministry of Education, Tongji University, Shanghai 201203, China Email:
        {\tt\small dzheng46@tongji.edu.cn; wangzhipeng@tongji.edu.cn; yuxilu@tongji.edu.cn; 2310881@tongji.edu.cn; binhe@tongji.edu.cn}}
\thanks{Hesheng Wang is with the Department of Automation, Shanghai Jiao Tong University, Shanghai 200240, China. Email:
        {\tt\small wanghesheng@sjtu.edu.cn}}
\thanks{Panagiotis Tsiotras is with the School of Aerospace Engineering and Institute for Robotics and Intelligent Machines, Georgia Institute of Technology, Atlanta, GA 30332, USA. Email:
        {\tt\small tsiotras@gatech.edu}}%
}

\providecommand{\IEEEkeywords}[1]{\textbf{\textit{Index terms---}} #1}


\maketitle

\begin{abstract}
Employing multiple manipulators can boost efficiency and accomplish tasks that a single manipulator cannot do. However, real-time planning for multiple manipulators in a cluttered workspace still poses significant challenges for planning algorithms. 
This paper proposes a new planning algorithm called Conflict-Based Lazy Search (CBLS) for multi-manipulator planning.
CBLS is built on Conflict-Based Search (CBS), an efficient multi-agent pathfinding (MAPF) algorithm that has shown an order of magnitude speedup over previous approaches~\cite{Sharon2015CBS, Andreychuk2022Multi}.  
CBS addresses MAPF by solving many single-agent pathfinding (SAPF) problems. Thus, its planning time directly depends on the efficiency of the SAPF algorithm adopted.
Our CBLS algorithm enhances CBS with precomputation and lazy search.
First, a lazily evaluated graph with controlled sparsity is precomputed for a single manipulator. 
Second, we propose the Lazy Edged-based A* (LEA*) for efficient SAPF.
Since edge evaluation is the computational bottleneck of manipulator planning,
LEA* uses lazy search and an edge queue to reduce the number of edge evaluations.
We show that LEA* is optimally vertex efficient and has improved edge efficiency compared to A*.
We apply the proposed CBLS to multi-manipulator planning problems and show its superior performance by comparing it with CBS and a sampling-based algorithm, namely, RRT-Connect.

\end{abstract}

\begin{IEEEkeywords}
Path planning, Conflict-based search, Lazy search, Manipulator planning, Multi-manipulator.
\end{IEEEkeywords}

\section{Introduction}

\IEEEPARstart{F}{inding} optimal collision-free paths for multi-agent systems is an NP-hard problem, especially for multi-manipulator systems operating in cluttered environments.
Sampling-based motion planning methods such as PRM \cite{kavraki1996probabilistic}, RRT*~\cite{Karaman2011Sampling}, RRT\# \cite{Arslan2013Use}, and BIT* \cite{gammell2020batch} make the problem tractable by approximating the search space using graphs or trees. 
They can also be viewed as graph-based methods. 
Graph search is used in PRM to find resolution-optimal paths, and local graph exploitation is used in RRT* to find asymptotically optimal paths.  
PRM builds a roadmap by drawing samples to cover the search space approximately and connecting neighboring samples using edges.  
RRT* builds a rapidly exploring tree by connecting the tree to incrementally added samples.

\begin{figure}[t]
    \centering
    \includegraphics[width=0.88\columnwidth]{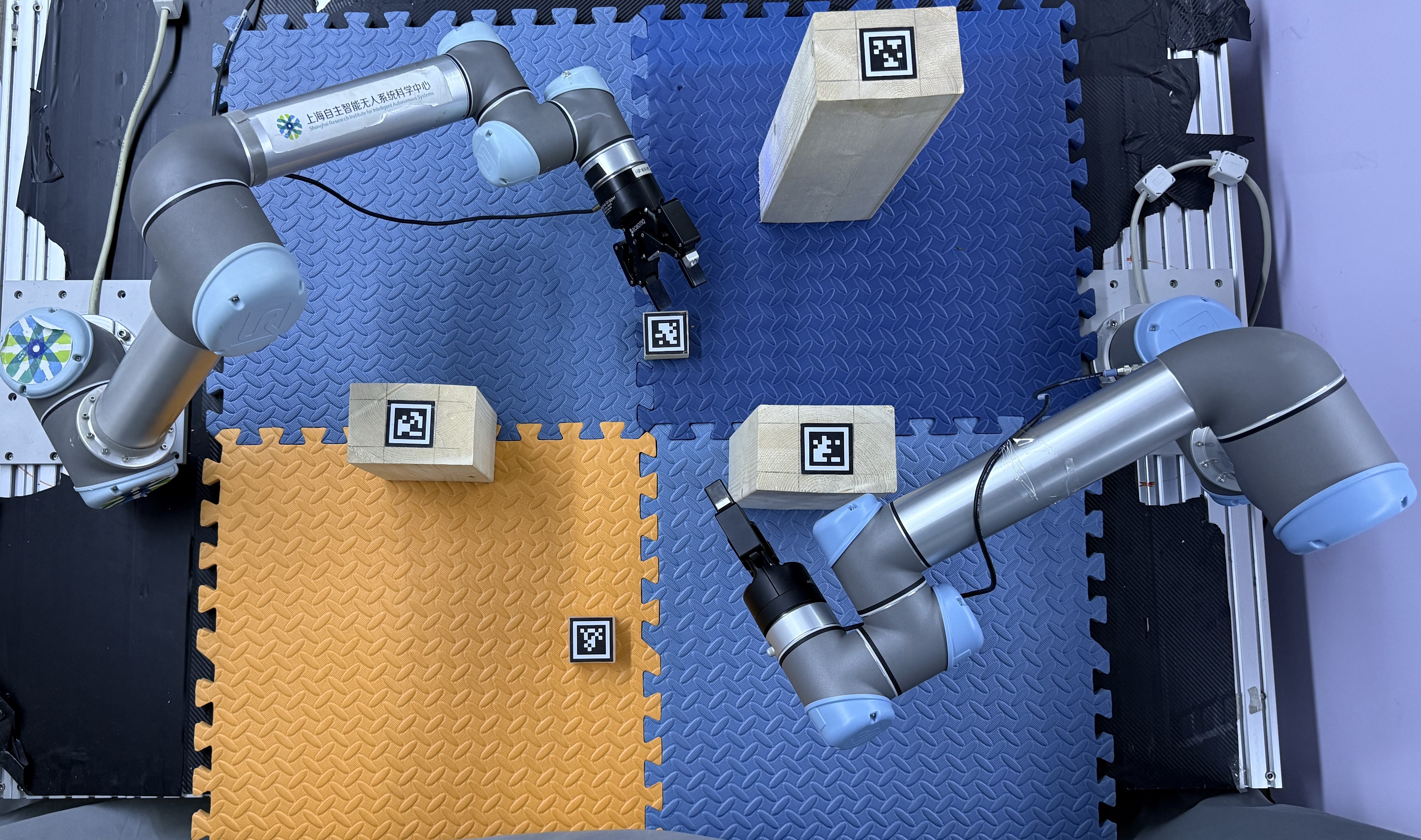}
    \caption{Manipulator planning setup. Given a valid start-goal configuration pair of the manipulators, the CBLS algorithm generates a collision-free joint angle path. A calibrated external camera is used to track the locations of the obstacles.}
    \label{Experiment_Setting}
\end{figure}

Sampling-based motion planning methods are well-suited for planning in high-dimensional spaces.
They find approximate solutions quickly by using graphs with a small number of samples. 
Earlier studies for multi-agent path planning combine the search space of each agent to form a `joint agent' and solve a single-agent planning problem~\cite{Shome2020dRRT, Cohen2014Single}.
The dimensionality of the search space is the sum of the dimensionality of each agent. 
However, the complexity of planning grows exponentially with the dimensionality of the search space.
These methods usually find solutions that are far from optimal, and the convergence rate to the optimal solution is slow.

Conflict-Based Search (CBS) \cite{Sharon2015CBS, Shaoul2024Un} is a dedicated search-based algorithm for multi-agent pathfinding (MAPF).
Instead of treating multiple agents as one `joint agent', CBS iterates between: a) independent single-agent pathfinding (SAPF); b)  conflict checking among all agent paths; and, c) conflict resolution.
CBS uses a two-level scheme. 
At the high-level, CBS checks for any conflict among single-agent paths and resolves conflicts by adding constraints.
At the low-level, single-agent paths consistent with the constraints are computed. 

CBS requires an existing graph for graph search and usually assumes the graph is a grid world.
For multi-manipulator planning using CBS, the first step is to construct a graph. 
Implicit graphs can be defined using motion primitives (MPs)~\cite{Cohen2011Planning, Saxena2021Mani}.
MPs specify the allowable actions of the manipulator.
The manipulator moves from one vertex to its neighboring vertex by applying such a motion primitive.
Using a graph search algorithm, \cite{Cohen2014Single} finds a sequence of MPs that connect the start and goal configuration of the manipulator.
Other works, such as PRM, use an explicit graph, where graph vertices are random samples drawn in the configuration space, and neighboring vertices within a radius are connected using edges.
In this paper, we construct an explicit graph for manipulator planning, and
introduce a parameter to control the sparsity of the graph.
The proposed graph construction method generates increasingly dense graphs for various requirements.
It can be applied to different numbers of manipulators and environments with different obstacle configurations and thus is environment-independent and robot-centric.

Graph search is performed for SAPF after graph construction.
Under the CBS framework, the MAPF planning time is proportional to the SAPF planning time~\cite{Sharon2015CBS}.
Numerous algorithms have been developed for SAPF. The A* \cite{Hart1968A} algorithm is a popular algorithm that is vertex optimal, that is, any other algorithm that finds the same shortest path will expand at least as many vertices when using the same heuristic. 
The weighted-A* algorithm finds a bounded suboptimal solution with an inflated heuristic cost \cite{Cohen2011Planning}.

One practical issue with robot path planning is that the edge costs may be unknown when starting the graph search. 
For example, in unknown environments, it is not known whether the edges are in collision with obstacles or not.
The procedure of computing the cost of an edge is called \textit{edge evaluation}.
Edge evaluations are performed online as part of the graph search algorithm to compute the edge cost (e.g., perform collision checking).
Even if the environment is known before starting the graph search, evaluating all edges before starting the graph search is unnecessary, as only part of the edges will be visited when searching for the solution. 
Thus, performing edge evaluation online saves time compared to evaluating all edges before graph search.
In many robotic motion planning problems, edge evaluation (collision checking) is the major computational bottleneck \cite{LaValle2006Planning}. 
Lazy search algorithms that aim to reduce the number of edge evaluations have been developed \cite{Cohen2015Planning, Dellin2016A}. 
Lazy search algorithms adopt a heuristic edge cost. 
Such a cost provides a lower bound of the true edge cost but is easier to compute. 
This heuristic cost is used to guide the graph search so that edge evaluation is done only when necessary.
LazySP is proven to be optimal in the sense that minimizes edge evaluations~\cite{Mandalika2018Lazy}.

Lazy search algorithms reduce edge evaluations at the expense of extra graph operations (vertex expansion, calculating the current best plan), which introduce extra computational overhead compared to A*.
This paper proposes a new efficient graph search algorithm, namely, LEA*, for SAPF. 
LEA* is vertex optimal and has improved edge efficiency compared to A*. 
It introduces less computational overhead compared to other lazy search algorithms.

The contributions of this paper are summarized as follows:
\begin{itemize}
    \item We propose CBLS, a new MAPF algorithm focusing on multi-manipulator path planning.
    CBLS enhances CBS with a precomputed, lazily evaluated graph and a new SAPF algorithm using lazy search. 
    
    \item We propose the LEA* algorithm for solving SAPF problems. LEA* uses lazy search and an edge queue to reduce the number of edge evaluations. 
    We show the completeness, optimality, and optimal vertex efficiency of LEA*. We also show the improved edge efficiency over the A* algorithm through numerical comparisons.
    
    \item A lazily-evaluated graph for efficient online multi-manipulator planning. A joint space graph with controlled sparsity is precomputed for manipulator planning. 
    
    \item Detailed simulation and experiment studies of the proposed algorithm are conducted to compare LEA* with previous SAPF algorithms. 
    We also compare CBLS with CBS and RRT-connect~\cite{Kuffner2000RRT-connect} demonstrating CBLS's superior performance. 
\end{itemize}

\section{Related Works} \label{SecRelatedWorks}

\subsection{Lazy Search for SAPF}

Graphs offer a powerful abstraction tool for robot path and motion planning.
When combined with sampling-based methods, algorithms such as PRM* \cite{Karaman2011Sampling}, RRT* \cite{Karaman2011Sampling, Jiang2022Path}, and BIT* \cite{gammell2020batch} can solve planning problems in high-dimensional spaces.
They construct an explicit graph of robot configurations and find the shortest path by exploiting this graph.
Implicit graphs may be defined using state lattices or motion primitives \cite{Likhachev2009Planning, Pivtoraiko2009Differentially, Liu2017Search}. Motion primitives are precomputed and beneficial in dealing with kinematic, nonholonomic, and differential constraints. 
Planning problems for ground vehicles and micro aerial vehicles are studied in \cite{Pivtoraiko2009Differentially} and \cite{Liu2017Search}, respectively. 

The A* algorithm is a popular search algorithm that is vertex optimal \cite{Hart1968A} but may lead to excessive edge evaluations.
Recent studies show that the number of edge evaluations can be reduced by employing a lazy approach \cite{bohlin2000path, Cohen2015Planning, Dellin2016A, Hauser2015Lazy}. 
For instance, LWA* \cite{Cohen2015Planning} uses a one-step lookahead to postpone edge evaluation and uses duplicated vertices in the vertex queue. First, a valid vertex with an estimated cost is popped from the vertex queue. Then, after edge evaluation, it is inserted into the queue again with the true cost.
LazySP \cite{Dellin2016A} uses an infinite-step lookahead and is shown to be edge optimal (evaluating the minimum number of edges). 
LRA* \cite{Mandalika2018Lazy} interpolates between LWA* and LazySP and uses a constant lookahead in the interval $[1, \infty]$.
As the lookahead steps increase, the number of edge evaluations required to find the shortest path decreases, while the additional graph operations increase. 
These lazy search algorithms achieve better edge efficiency at the cost of considerable overhead compared to A*. 

In \cite{Mandalika2019Generalized}, the generalized lazy search algorithm (GLS) that unifies LWA*, LazySP, and LRA* was proposed.
GLS introduces an EVENT function and a SELECTOR function.
It admits various realizations through different choices of EVENT and SELECTOR functions.
With specific EVENT and SELECTOR instances, GLS reduces exactly to the LazySP or LRA* algorithm.
GLS may employ informed EVENT and SELECTOR functions to improve the planning efficiency of LazySP/LRA*.
However, this requires additional prior information (e.g., priors on edge validity learned from experience), which may be unavailable or may deviate from the true edge collision probabilities.
A lazy incremental algorithm dealing with dynamically changing graphs was introduced in \cite{Lim2022Lazy} by combining the idea of lazy search and lifelong planning \cite{Koenig2004Lifelong}. 

\subsection{MAPF and Conflict-Based Search}

To solve MAPF problems, earlier works search the augmented space of all agents.
The single and dual-arm planning problem was studied in~\cite{Cohen2014Single}. 
It uses the ARA* algorithm to search an implicit graph defined by motion primitives.
By combining individual agents as a single ‘joint agent’, the dimensionality of planning space becomes too large.
These methods may find solutions that are far from the optimal one, and the convergence rate to the optimal solution is slow.

Conflict-based search (CBS)~\cite{Sharon2015CBS} uses a two-level planning paradigm. 
It uses the fact that SAPF is much simpler than MAPF.
The low-level of CBS solves many SAPF problems and uses high-level planning to resolve conflicts between agents. 
Several enhancements of CBS have been proposed where admissible heuristics are developed to guide the high-level search \cite{Boyarshi2015ICBS, Li2019Improved}.
Disjoint splitting for conflict resolution was proposed in \cite{Li2019Disjoint} to reduce duplication of search efforts.
The CBS-Budget (CBSB) algorithm was proposed in \cite{Lim2022CBSB}. A Class-Ordered A* algorithm, which finds the shortest path with a minimal number of conflicts that is upper bounded in terms of length, is used for the low-level planning of CBSB.
Continuous-time Conflict-Based Search (CCBS) was proposed in~\cite{Andreychuk2022Multi} to remove time discretization, where Safe Interval Path Planning (SIPP) was used for single-agent pathfinding in CBS.

All methods mentioned above consider a grid world. 
Multi-robot kinodynamic motion planning was studied in \cite{Moldagalieva2024db} based on CBS and discontinuity-bounded A*. The proposed method is used for multiple vehicle planning. 
Multi-manipulator planning was studied in \cite{Shaoul2024Un}. Incomplete constraints that allow fast search space pruning are introduced to prune the search space of CBS. 
The resulting generalized ECBS algorithm obtains bounded sub-optimality.
Note that the method proposed in this paper and \cite{Shaoul2024Un} are complementary.
Reference \cite{Shaoul2024Un} uses incomplete constraints as a heuristic to reduce the number of SAPF problems needed by CBS, and uses weighted-A* for SAPF \cite{Cohen2011Planning}, while our method mainly focuses on more efficient SAPF methods. 
Our paper uses complete constraints (vertex and edge constraints); incomplete constraints could be integrated following \cite{Shaoul2024Un}.

\subsection{Manipulator Planning}

Manipulator motion planning for human-manipulator interaction and collaboration is studied in \cite{Liu2024Integrating}, where human motion prediction is utilized for proactive manipulator planning.
In \cite{Jang2022Motion}, the multi-arm collaborative pick-and-place task is studied.
The closed-chain constraint is considered using a probabilistic roadmap algorithm. The task considers a relatively structured environment, reducing the complexity of obstacle avoidance. 
Prioritized planning is used in \cite{Shi2022Obstacle}, where a dual arm is divided into a main arm and a slave arm. 
Based on the RRT algorithm, the main arm planning is performed first, then the main arm is treated as a dynamic obstacle with a known trajectory in the slave arm planning.
An asymptotically optimal multi-robot motion planning algorithm was proposed in \cite{Shome2020dRRT}, where an informed search is performed on the tensor product of roadmaps.
Multi-manipulator planning using CBS is studied in \cite{shaoul2024Acc}, where online-generated experience is exploited to reduce single-manipulator planning time at the low level of CBS.

\section{Problem Fromulations} \label{Sec:Problem}
Given $n$ manipulators working in a shared environment, we label them as $a_1, \ldots, a_n$.
Let the joint angles of manipulator $i$ at time $k$ be $j_{i,k}$, where $j_{i,k}$ is a vector with dimension equal to the number of joints of manipulator $i$.
The set of joint angles for all $n$ manipulators at time $k$ is given by $J_k = \{j_{1,k}, j_{2,k}, \ldots, j_{n,k}\}$. The set $J_k$ is a \textit{valid joint angle set} if there is no manipulator self-collision, manipulator-obstacle collision, or manipulator-manipulator collision at time $k$.
A joint angle path for all $n$ manipulators $P = (J_1, \ldots, J_K )$ is a \textit{valid joint angle path} if the joint angle set $J_k$ is valid for all time $k = 1, \ldots, K$ along the path.

The multi-manipulator planning problem is defined as finding a valid path $P$ from a start joint angle set $J_s$ to a goal set $J_g$.
The cost of $P$ is the sum of all individual agent path lengths.
We adopt conflict-based search (CBS) \cite{Sharon2015CBS} to solve the multi-manipulator planning problem.
CBS is a two-level algorithm.
The low-level solves a series of SAPF problems and finds optimal paths for the individual agents.
For each agent $a_i$, SAPF involves searching a precomputed graph to find a valid shortest length path $p_i = \{j_{i,1}, \ldots, j_{i,K}\}$ from the start $j_{i,1}$ to goal $j_{i,K}$.
Each path is free from manipulator self-collisions and manipulator-obstacle collisions, but overlooks manipulator-manipulator collisions.
The high-level planner checks for conflicts among all individual paths. 
Any manipulator-manipulator collision results in a conflict.
If no conflict is found, a valid path has been found.
Otherwise, the conflict is split into new constraints.
Low-level SAPF problems with the new constraints are then solved.

We use lazy search to solve the low-level SAPF. 
Collision checking is one of the computational bottlenecks of manipulator planning. 
To reduce the total number of collision checks, lazy search methods use an edge cost heuristic before computing the true edge cost.
A graph $G$ is given by a vertex set $V$ and an edge set $E$.
We define the edge cost heuristic $\hat{w}(e)$ for an edge $e$ as the length of the edge.
The true edge cost $w(e)$ is given by
\begin{equation}
    w(e) = \begin{cases} \hat{w}(e),  \ \quad \text{if $e$ is not in collision}, \\ 
    \enspace \infty,  \ \ \ \quad \text{if $e$ is in collision}.
    \end{cases}
\end{equation}

In the following sections, we first summarize our method for constructing a graph for manipulator planning.
Then, the lazy search algorithm along with its properties is introduced.
Finally, we combine the two components with CBS to form the conflict-based lazy search (CBLS) algorithm.

\section{Lazily Evaluated Sparse Graph} \label{Sec:Graph}

CBLS uses a precomputed graph for online pathfinding.
Without loss of generality, we consider the path planning problem for several identical manipulators.
To deal with different manipulators, we need to compute one graph for each manipulator model. 

The vertex sampling method is given in Algorithm \ref{alg:SparseSampling}.
In Line 4, we draw a sample from the joint angle space of one manipulator using uniform random sampling.
This step is repeated until a self-collision-free sample $v$ is found (Lines 5-6).
We introduce the parameter $\textit{m\_dist}$ as the lower-bound distance between sampled vertices.
The $\mathsf{mindist}$ function returns the minimum distance $d$ between $v$ and all other vertices in $V$ (Line 7).
Vertex $v$ is rejected if $d$ is less than $\textit{m\_dist}$ (Lines 8-10).
Otherwise, $v$ is added to the vertex set $V$. 
We use $\textit{count}$ to keep track of the number of consecutive rejected vertices. 
In Line 3, the algorithm terminates if a maximum sample number $N$ is reached or $\textit{count}$ reaches $\textit{max\_count}$.

%%%%%%%%%%%%%%%%%%%%%%%%%%%%%%%%%%%%%%%%%%%%%%%%%%%%%%%%%%%%%%%%%%%%%%%%%%%
\IncMargin{.5em}
\begin{algorithm}
\caption{SparseSampling}
\label{alg:SparseSampling}
\KwIn{parameters $N$, $\mathrm{max\_count}$, $\mathrm{m\_dist}$}
\KwOut{vertex set $V$}
$V \leftarrow \emptyset$\;
$\mathrm{count} \leftarrow 0$\;
\While{$ V.\mathsf{size} < N \ \mathbf{and} \ \mathrm{count} < \mathrm{max\_count}$}
{   
    $v \leftarrow \mathsf{Sampling}$\;
    \If{$\mathsf{selfCollision}(v)$} 
    {
        $\mathsf{continue}$\;
    }
    $d \leftarrow \mathsf{mindist}(V, v)$\;
    \If{$d < \mathrm{m\_dist}$} 
    {
        $\mathrm{count} = \mathrm{count} + 1$\;
        $\mathsf{continue}$\;
    }
    $\mathrm{count} \leftarrow 0$\;
    $V \leftarrow V \cup \{v\}$
}
$\mathsf{return} \ V$\;
\end{algorithm}
\DecMargin{.5em}
%%%%%%%%%%%%%%%%%%%%%%%%%%%%%%%%%%%%%%%%%%%%%%%%%%%%%%%%%%%%%%%%%%%%%%%%%%%

The parameter $\textit{m\_dist}$ controls the sparsity of the vertex set, where all vertices are at least $\textit{m\_dist}$ away from each other.
Starting with a large $\textit{m\_dist}$, we use a small number of samples to cover the joint angle space of the manipulator. 
Using the most recent vertex set $V$ and gradually reducing $\textit{m\_dist}$, we obtain an incrementally denser graph using Algorithm \ref{alg:SparseSampling}, where the vertices are still approximately equally spaced.
Algorithm \ref{alg:SparseSampling} is only for vertex sampling.
To construct a graph, we connect each vertex to its neighbors following the PRM* algorithm.

We check self-collisions during graph construction, while manipulator–obstacle and manipulator–manipulator collisions are handled via online lazy search.
Because specific obstacles are not considered during graph construction, the resulting graph is environment-independent and can be reused across environments with different obstacle configurations.
We exploit offline computation and online lazy search to reduce online computation.
Other online algorithms, such as RRT* and motion primitive-based algorithms, also need to perform online self-collision checking.
Previous MAPF formulations often assume a grid world that is fixed to a global frame, and all agents move within this grid world.
In contrast, in our manipulator planning setting, the graph is agent-centric.
Each manipulator has its own graph. Two manipulators will occupy different workspaces even if they have the same joint angles.

\section{The LEA* Algorithm for SAPF} \label{LEA*}

In this section, we provide a detailed description of the proposed lazy edge-based A* (LEA*) algorithm.
LEA* is an efficient algorithm for SAPF, which, in turn, helps reduce the planning time for CBS.
LEA* uses lazy search and edge queue for tree expansion.
We design LEA* to have minimal changes compared to the original A* algorithm, thus having restrained computational overhead, by searching the lazily evaluated graph (Section \ref{Sec:Graph}).

Similar to A*, we define the \textit{cost-to-come} and the heuristic \textit{cost-to-go} of the vertex $v$ as $g(v)$ and $h(v)$, which are the cost from $v_s$ to $v$ given the current search tree and the estimated cost from $v$ to $v_g$, respectively. Here, $h(\cdot)$ is an admissible and consistent heuristic.
The total estimated cost of the path passing through $v$ is given by 
\begin{equation}
    f(v) = g(v) + h(v).
    \label{f_v}
\end{equation}
An edge is given by $e = (v_0, v_1)$, where $v_0$ is the source vertex and $v_1$ is the target vertex.
We define the total estimated cost of the path passing through edge $e$ as
\begin{equation}
    f(e) = g(v_0) + \hat{w}(e) + h(v_1).
    \label{f_e}
\end{equation}

\begin{figure*}[t]
\begin{center}
    \begin{subfigure}{0.26\linewidth}{
    \includegraphics[width=\textwidth]{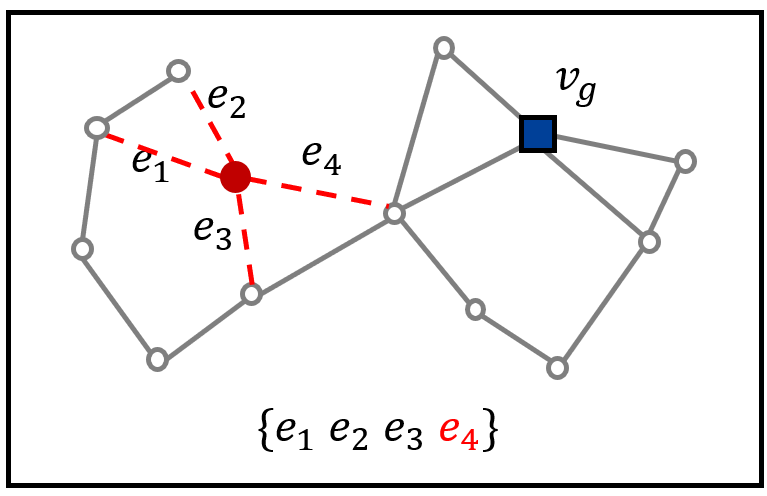}
    \caption{}
    }
    \end{subfigure}
    \begin{subfigure}{0.26\linewidth}{
    \includegraphics[width=\textwidth]{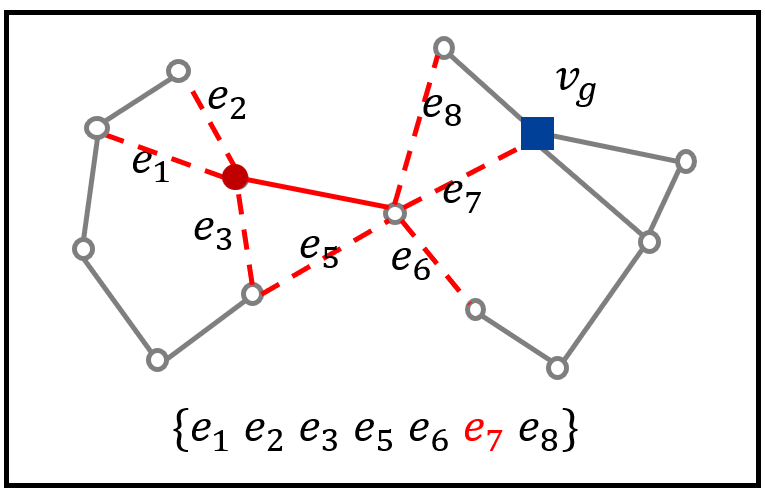}
    \caption{}
    }
    \end{subfigure}
    \begin{subfigure}{0.26\linewidth}{
    \includegraphics[width=\textwidth]{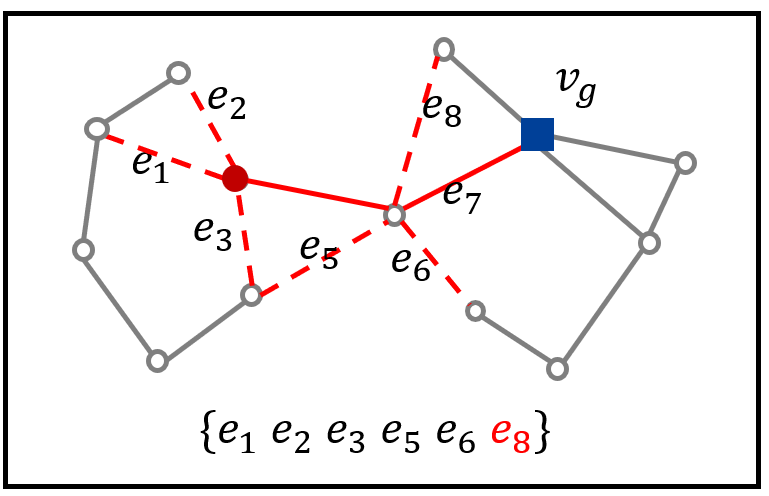}
    \caption{}
    }
    \end{subfigure}
    \begin{subfigure}{0.26\linewidth}{
    \includegraphics[width=\textwidth]{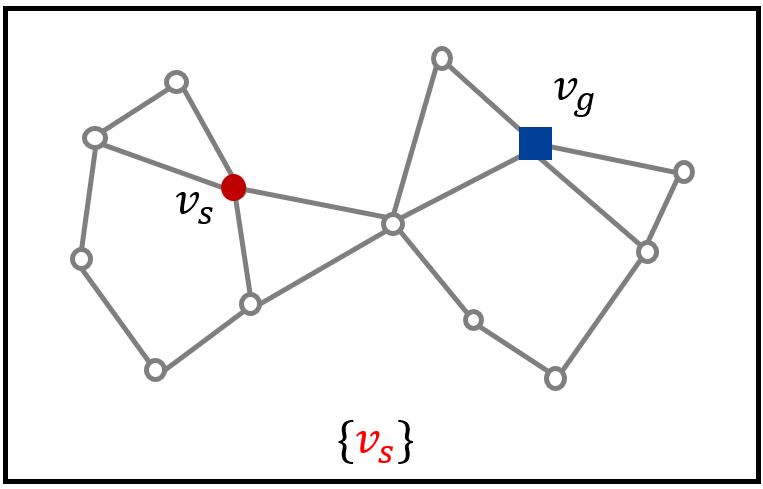}
    \caption{}
    }
    \end{subfigure}
    \begin{subfigure}{0.26\linewidth}{
    \includegraphics[width=\textwidth]{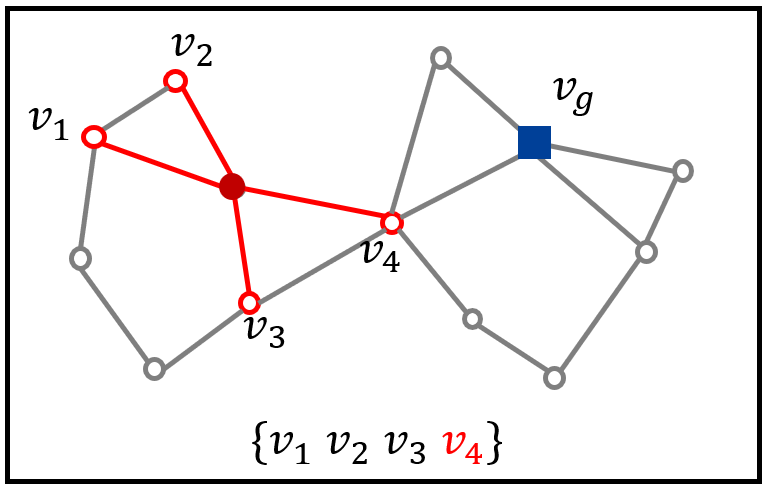}
    \caption{}
    }
    \end{subfigure}
    \begin{subfigure}{0.26\linewidth}{
    \includegraphics[width=\textwidth]{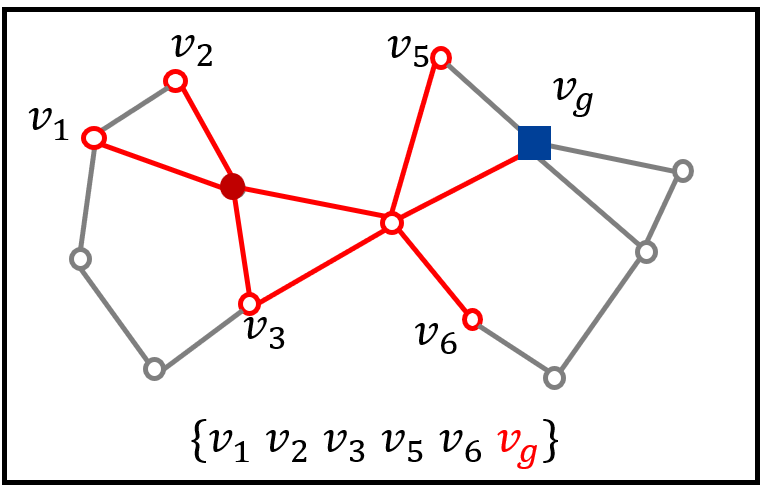}
    \caption{}
    }
    \end{subfigure}
    \caption{Illustration of the steps of LEA* ((a)-(c)) and A* ((d)-(f)). The set below each figure is the current edge queue or vertex queue. Dashed red lines are edges that are lazily evaluated. Solid red lines are evaluated edges. In (a), four edges are added to the edge queue, and $e_4$ is selected for evaluation. In (b), $e_4$ is evaluated, the next edges are added to the edge queue, and $e_7$ is selected. By using lazy evaluation and an edge queue, LEA* only evaluated two edges while A* evaluated eight edges.}
  \label{fig:LEA*illustration}
\end{center}
\end{figure*}

%%%%%%%%%%%%%%%%%%%%%%%%%%%%%%%%%%%%%%%%%%%%%%%%%%%%%%%%%%%%%%%%%%%%%%%%%%%
\IncMargin{.5em}
\begin{algorithm}
\caption{LEA*}
\label{alg:LEA*}
\KwIn{start $v_s$, goal $v_g$ for one arm}
\KwOut{planned path for one arm}
$Q_E \leftarrow \{(v_s, v_i)| \forall v_i \in \mathsf{Succ}(v_s) \}$\;
\While{$Q_E \neq \emptyset$}
{   
    $(v_0, v_1) \leftarrow Q_E.\mathsf{Pop}$\;
    \If{$g(v_g) \leq f((v_0, v_1))$}
    {
        $\mathsf{return} \ \mathsf{RetrievePath}$
    }
    % \If{$(v_0, v_1) \in c$}
    % {
    %     $\mathsf{continue}$
    % }
    $\mathrm{collisionFree} \leftarrow \mathsf{CollisionCheck}((v_0, v_1))$\;
    \If{$\mathrm{collisionFree}$}
    {
        $g_\mathrm{new} \leftarrow g(v_0) + w((v_0, v_1))$\;
        \If{$g_\mathrm{new}<g(v_1)$}
        {
            $g(v_1) \leftarrow g_\mathrm{new}$\;
            $v_1.\mathrm{parent} \leftarrow v_0$\;
            $Q_E \leftarrow Q_E \cup \{(v_1, v_i)|\forall v_i \in \mathsf{Succ}(v_1) \}$\;        
        }
    }
}
\end{algorithm}
\DecMargin{.5em}
%%%%%%%%%%%%%%%%%%%%%%%%%%%%%%%%%%%%%%%%%%%%%%%%%%%%%%%%%%%%%%%%%%%%%%%%%%%

The LEA* algorithm is given in Algorithm \ref{alg:LEA*}.
In Line~1, the edges from the starting vertex $v_s$ to its successors $v_i$ are added to the edge queue $Q_E$.
The edges in $Q_E$ are ordered by their $f$-values given by equation (\ref{f_e}).
In Line 3, the best edge $e=(v_0, v_1)$ with the smallest $f$-value is popped from $Q_E$. 
The $f$-value is used in the termination condition in Line 4.
If the cost-to-come of goal vertex $g(v_g)$ is less than or equal to $f((v_0, v_1))$, we have found the shortest path.
We can retrieve the shortest path from $v_s$ to $v_g$ using the parent information of the vertices.
Collision checking is performed in Line 6. 
If the edge is collision-free, i.e., no manipulator self-collision or manipulator-obstacle collision, the new cost-to-come of $v_1$, $g_\mathrm{new}$, is computed in Line 8.
If $g_\mathrm{new}$ is better than the current cost-to-come $g(v_1)$, $g(v_1)$ and the parent of $v_1$ are updated (Line 9-11).
Finally, in Line 12, the edges from $v_1$ to its successors $v_i$ are added to the edge queue $Q_E$.

A simple illustration of the LEA* and its difference with the A* algorithm is given in Figure~\ref{fig:LEA*illustration}.
Figures~\ref{fig:LEA*illustration}(a)-(c) are the steps of LEA* and Figures~\ref{fig:LEA*illustration}(d)-(f) are the steps of A*.
At the beginning of LEA*, four edges are added to the edge queue as shown in Figure~\ref{fig:LEA*illustration}(a). 
Note that these edges are only lazily evaluated using the estimated cost $\hat{w}$, and $\hat{w}$ is used to compute their $f$-values according to (\ref{f_e}). 
The best edge $e_4$ is selected for edge evaluation. 
In Figure~\ref{fig:LEA*illustration}(b), $e_4$ is evaluated using collision checking, and the next edges are added to the queue. 
Repeating this process, the best edge $e_7$ in the current queue is selected for evaluation.
For A*, all outgoing edges are evaluated when expanding a vertex.
In this example, by using lazy evaluation and an edge queue, LEA* only evaluated two edges to find the solution, while A* evaluated eight edges.

\subsection{Algorithm Analysis}

In this subsection, we analyze the completeness, optimality, and vertex efficiency of LEA*. We first show that LEA* is complete.
\begin{proposition}
If at least one solution exists for the single-agent graph search problem, LEA* will return a solution. Otherwise, it will return that no solution exists.
\end{proposition}
\begin{proof}
We first show that the algorithm terminates with $Q_E = \emptyset$ and $g(v_g) = \infty$ when no solution exists.
All edges have positive edge costs. The new edge $(v_1, v_i)$ is added to $Q_E$ in Line 12, Algorithm~\ref{alg:LEA*}, only when we find a better path to $v_1$ (Line 9). Since the cost-to-come to every vertex is decreasing and lower bounded, Line 12 will only run a finite number of iterations. Thus, $Q_E = \emptyset$ after executing Line 3 a finite number of iterations. Therefore, LEA* will terminate with $Q_E = \emptyset$ and $g(v_g) = \infty$ if no solution exists.

Next, we consider the case when the problem has a solution.
LEA* terminates when $g(v_g) \leq \mathrm{min}_{e \in Q_E} f(e)$ (Line 4) or $Q_E = \emptyset$ (Line 2).
If the algorithm terminated with $g(v_g) \leq \mathrm{min}_{e \in Q_E} f(e)$, we have $g(v_g) < \infty$, which implies a solution has been found.

Let $(e_0,e_1, \dots, e_{n-1})$ (equivalently, $(v_s, v_1, v_2, \dots, v_g)$) be a solution path. 
Initially, $e_0 \in Q_E$. 
After evaluating $e_0$, $e_1$ is added to $Q_E$. 
Similarly, after evaluating $e_1$, edge $e_2$ is added to $Q_E$.
If the algorithm terminateds with $Q_E = \emptyset$, it must have evaluated $e_0,e_1, \dots, e_{n-1}$. Thus, the algorithm has found this solution, and the returned solution is at least as good as $(e_0,e_1, \dots, e_{n-1})$.  
\end{proof}

Next, we prove the optimality properties of LEA*.
\begin{proposition}
If at least one solution exists for the graph search problem, LEA* finds the minimum cost solution.
\end{proposition}
\begin{proof}
Let $\tau^* = (v_s, v_1^*, v_2^*, \dots, v_{n-1}^*, v_g)$ (equivalently $(e_0^*, e_1^*, \dots, e_{n-2}^*, e_{n-1}^*)$) be an optimal solution path with path cost $c^*$. 
Let $\tau = (v_s, v_1, v_2, \dots, v_{n-1}, v_g)$ (equivalently $(e_0, e_1, \dots, e_{n-2}, e_{n-1})$) be the solution path returned by LEA*. 
The cost of $\tau$ is $c$. 
To prove optimality, we need to show that $c$ is equal to $c^*$.

To reach a contradiction, let us assume $c^* < c$.
Note that $\max\{f(e_0^*), f(e_1^*), \dots, f(e_{n-2}^*), f(e_{n-1}^*)\} \leq c^*$. 
Also,  $g(v_g) = c \leq \mathrm{min}_{e \in Q_E} f(e)$ holds when LEA* terminates.
At the beginning of LEA*, $e_0^* \in Q_E$. Therefore, $e_0^*$ must have been evaluated before $c \leq \mathrm{min}_{e \in Q_E} f(e)$ is true. 
After evaluating $e_0^*$, the vertex $v_1^*$ is added to the expansion tree, and edge $e_1^*$ is added to $Q_E$. 
By repeating this analysis, $e_1^*, \dots, e_{n-2}^*, e_{n-1}^*$ must all have been evaluated before $c \leq \mathrm{min}_{e \in Q_E} f(e)$ is true.
After evaluating $e_1^*, \ldots, e_{n-1}^*$, we have found a path to $v_g$ and $f(v_g) \leq c^*$. Since the $f(v_g)$ is nonincreasing, the algorithm will never return a path with cost $c>c^*$. Therefore $c=c^*$.
\end{proof}

The next result illustrates the optimal vertex efficiency property of LEA*.
\begin{proposition}
LEA* has the same vertex efficiency as A*. Furthermore, the evaluated edge set of LEA* is a subset of the evaluated edge set of A*.    
\end{proposition}
\begin{proof}
For LEA*, we call evaluating any outgoing edge of vertex $v$ as ``expanding $v$."
A* expands $v$ by evaluating all outgoing edges of $v$. Vertex $v$ is expanded in LEA* if a subset of its outgoing edges is evaluated.

From the previous section, we know that LEA* finds an optimal path. 
Let $\tau^* = (e_0^*, e_1^*, \dots, e_{n-2}^*, e_{n-1}^*)$ be the path found by LEA*, and $c^*$ be the cost of $\tau^*$.
Let $T=(V,E)$ and $T^e = (V^e, E^e)$ be the expansion tree grown by A* and LEA*, respectively.
Let $W^v$ and $W^e$ be the set of vertices expanded by A* and LEA*, respectively.
By showing $W^e \subseteq W^v$, we prove the optimally efficient search of 
LEA*.
Note that if $W^e \subseteq W^v$, we have $V^e \subseteq V$. 

Assume $W^e \not\subseteq W^v$. Then, there exists $v$, such that $v \notin W^v$, $v \in W^e$, $v \in V^e$ and $v \in V$. 
To show this, we start with $v_0 \in W^e$, and $v_0 \notin W^v$. Then, $v_0 \in V^e$ (vertex must be in the tree for it to be expanded). If $v_0 \notin V$, we find $v_1$, which is the parent of $v$ in $T^e$. Then, we have $v_1 \in V^e$, $v_1 \in W^e$, and $v_1 \notin W^v$. If $v_1 \notin V$, we repeat this process by finding its parent vertex in $V^e$, and one of these parents $v_i$ must satisfy $v_i \in V$ since $V$ and $V^e$ share the same root vertex.  

From $v \notin W^v$, we have $c^* \leq f(v) = g(v) + h(v)$.
From $v \in W^e$, we have that there exists $v_i \in \mathsf{Succ}(v)$, such that $f((v,v_i)) = g(v)+\hat{w}((v,v_i)) +h(v_i) \leq c^*$. Note that $f(v) \leq f((v,v_i))$. Then, we have
\begin{equation}
    c^* \leq f(v) \leq f((v,v_i)) \leq c^*,
\end{equation}
which only holds when $f(v) = f((v,v_i)) = c^*$.
Note that
\begin{equation}
    \max\{f(e_0^*), f(e_1^*), \dots, f(e_{n-2}^*), f(e_{n-1}^*)\} \leq c^*. \label{eq1}
\end{equation}
Using (\ref{eq1}), in order for $(v, v_i)$ to be evaluated by LEA*, there exist $e_i^* = (v_0, v_1) \in \tau^*$, such that $f(e_i^*) = c^*$ and $g(v) < g(v_0)$. Otherwise, all $e_i^* \in \tau^*$ have a higher priority than $(v, v_i)$ and $(v, v_i)$ will not be evaluated. 
Using $f(v) = c^* $ and $g(v) < g(v_0) \leq g(v_g)$, $v$ has a higher priority than $v_g$. Therefore, $v$ must be expanded by A* before $v_g$ is selected from the vertex queue. Therefore, $v \in W^v$, contradicting $v \notin W^v$. Thus, $W^e \subseteq W^v$ and LEA* is optimally vertex efficient. 

Since A* evaluates all outgoing edges of $v \in W^v$ and LEA* evaluates a subset of outgoing edges of $v \in W^e$, LEA* evaluates fewer edges than A*.
\end{proof}

LEA* is an efficient SAPF algorithm. To integrate it into CBLS, a time dimension is added, similar to time-expanded A* \cite{Silver2005Cooperative}, and a constraint set is added.
Denote the constraint set as $\mathcal{C}$, which will be clarified in Section \ref{Sec:CBLS}. We add two lines after Line 5 to check if $(v_0, v_1) \in \mathcal{C}$ or not. If $(v_0, v_1) \in \mathcal{C}$, this edge has a conflict with other manipulators and is skipped. The algorithm moves to its next iteration.

\section{The CBLS Algorithm for MAFP}  \label{Sec:CBLS}
%%%%%%%%%%%%%%%%%%%%%%%%%%%%%%%%%%%%%%%%%%%%%%%%%%%%%%%%%%%%%%%%%%%%%%%%%%%
\IncMargin{.5em}
\begin{algorithm}
\caption{CBLS}
\label{alg:CBLS}
\KwIn{start and goal queries for $n$ arms}
\KwOut{planned paths for $n$ arms}

\For{$i = 1, \ldots ,n$}
{   
    $\mathcal{C}_i \leftarrow \emptyset$\;
    $(s_i, \mathrm{cost}_i) = \mathsf{LEA^*}(\mathrm{arm}_i, \mathcal{C}_i)$\;
    $Root.\mathcal{C}_i \leftarrow \mathcal{C}_i; \ Root.s_i \leftarrow s_i$\;
    $Root.\mathrm{cost}_i = \mathrm{cost}_i$\;
}
$Q \leftarrow \{Root\}$ \;
\While{$Q \neq \emptyset$}
{   
    $P \leftarrow Q.\mathsf{pop}$\;
    $\mathrm{Conflict} \leftarrow \mathsf{CheckConflict}(P)$\;
    \If{$\mathrm{Conflict} = \emptyset$}
    {
        $\mathsf{return} \ P$
    }
    \For{$\mathrm{arm}_i \in \mathrm{Conflict}$}
    {
        $\mathcal{C}_i^{'} \leftarrow P.\mathcal{C}_i \cup \mathrm{Conflict}.\mathrm{arm}_i$\;
        $(s_i^{'}, \mathrm{cost}_i^{'}) = \mathsf{LEA^*}(\mathrm{arm}_i, \mathcal{C}_i^{'})$\;
        $P_\mathrm{new} \leftarrow P$\;
        $P_\mathrm{new}.\mathcal{C}_i \leftarrow \mathcal{C}_i^{'}; \ P_\mathrm{new}.s_i \leftarrow s_i^{'}$\;
        $P_\mathrm{new}.\mathrm{cost}_i \leftarrow \mathrm{cost}_i^{'}$\;
        $Q.\mathsf{push}(P_\mathrm{new})$\;
    }
}
\end{algorithm}
\DecMargin{.5em}
%%%%%%%%%%%%%%%%%%%%%%%%%%%%%%%%%%%%%%%%%%%%%%%%%%%%%%%%%%%%%%%%%%%%%%%%%%%
CBLS uses a precomputed graph and makes use of LEA* for MAPF.
It solves constrained SAPF problems for each manipulator, checks for any conflict between planned paths, and resolves the conflict by adding new constraints for SAPF problems. It repeats this process until a solution is found or no solution can be found. 

In the original CBS algorithm, vertex conflicts and edge conflicts are defined for grid world problems.
For manipulator planning, we treat a vertex conflict as a special case of an edge conflict.
Therefore, we only consider edge conflicts.
After solving the SAPF problems for each manipulator, we obtain the paths of agent $i$ and $k$ which are given by $p_i = \{j_{i,1}, j_{i,2}, \ldots, j_{i,L}\}$ and $p_k = \{j_{k,1}, j_{k,2}, \ldots, j_{k,M}\}$, respectively.
Note that joint angle $j$ is a vertex $v$ of the graph since we sample vertices in the joint angle space.
We assume that each segment of the path (edge) takes one unit of time to travel.
This assumption is reasonable as graph vertices are approximately equally spaced.
Then $p_i$ takes $L$ units of time, and $p_k$ takes $M$ units of time.
If $L>M$, we can always append an appropriate number of $j_{k,M}$ to $p_k$ such that $p_i$ and $p_k$ have the same length.
This is also true for the case $L<M$.

When checking for conflicts between $p_i$ and $p_k$, we sequentially check for any collisions between edges in $p_i$ and $p_k$. 
If a collision happens between $(j_{i,\ell}, j_{i,\ell+1})$ and $(j_{k,\ell}, j_{k,\ell+1})$, a conflict is formed given by $((j_{i,\ell}, j_{i,\ell+1}), (j_{k,\ell}, j_{k,\ell+1}), \ell)$.
This conflict states that agent $i$ and $k$ cannot move along the edge $(j_{i,\ell}, j_{i,\ell+1})$ and $(j_{k,\ell}, j_{k,\ell+1})$ simultaneously at time $\ell$.
To resolve this conflict, a new constraint $((j_{i,\ell}, j_{i,\ell+1}), \ell)$ is added to the constraint set $\mathcal{C}_i$ and a new constraint $((j_{k,\ell}, j_{k,\ell+1}), \ell)$  is added to $\mathcal{C}_k$.
In the next iteration of CBLS, SAPF problems with the updated constraint set are solved for agents $i$ and $k$.

The CBLS algorithm is given by Algorithm \ref{alg:CBLS}.
In Lines 1-5, individual SAPF problems for each agent are solved using LEA*.
Initially, the constraint set $c_i$ for each agent $i$ is empty. 
$s_i$ and $\mathrm{cost}_i$ are the planned path and path cost for agent $i$, respectively.
The nodes in the conflict tree save the planning solutions for all agents.
The initial solution is saved in the $Root$ node.
The queue $Q$ is initialized with the $Root$ node (Line 6).
The nodes in $Q$ are prioritized by the total path cost of all agents.
In Line 8, the node with the lowest cost is popped from $Q$.
Manipulator-manipulator collision checking for paths in $P$ is performed in Line 9.
If no conflict is found among the paths, a solution is returned (Lines 10-11).
Otherwise, the conflict is resolved in Lines 12-18.
For each $arm_i$ affected by the conflict, the constraint set is updated (Line 13), and a SAPF problem is solved for $arm_i$ using LEA* (Line 14). 
We update $P$ with the new solution of $arm_i$ (Lines 16-17), and a new node $P_\mathrm{new}$ is added to the queue $Q$ (Line 18).

\begin{table*}
\caption{Planning results for a 7DOF manipulator}\label{tab:7Dtime}
\centering
\begin{tabular}{L{1.6cm}|L{1.3cm}L{1.3cm}L{1.3cm}|L{1.3cm}L{1.3cm}L{1.3cm}|L{1.3cm}L{1.3cm}L{1.3cm}}
\toprule
\hline
\multirow{4}{1.6cm} { } & \multicolumn{3}{c|}{Planning time (s)} & \multicolumn{3}{c|}{Edge evaluations} & \multicolumn{3}{c}{Path length}  \\
\cline{2-10}  &  N=${1,000}$  & N=${5,000}$   &  N=${10,000}$   &  N=${1,000}$   & N=${5,000}$   &  N=${10,000}$ &  N=${1,000}$   & N=${5,000}$   &  N=${10,000}$     \\
\hline     
A* ($\varepsilon$=1)      & 0.6773 & 1.7215   & 3.5747  & 331.29  & 917.36 & 1929.2 & 23.9935 & 22.1431 & \textbf{21.5779} \\
LazySP ($\varepsilon$=1)  & 0.1060 & 0.3527   & 0.6567  & \textbf{15.77}   & \textbf{15.11}  & \textbf{17.69}  & 23.9935 & 22.1431 & \textbf{21.5779} \\
LEA* ($\varepsilon$=1)    & \textbf{0.0763}   & \textbf{0.1123}   & \textbf{0.1755}  & 51.45   & 82.55  & 129.48  & 23.9935 & 22.1431 & \textbf{21.5779} \\
\hline
A* ($\varepsilon$=2)      & 0.2633 & 0.3800   & 0.4879  & 147.34  & 235.69 & 309.07   & 24.4210 & 22.4675 & 21.7728 \\
LazySP ($\varepsilon$=2)  & 0.0431 & 0.0433   & 0.0472  & \textbf{15.12}   & \textbf{13.20}  & \textbf{14.61}  & 24.4210 & 22.4675 & 21.7728  \\
LEA* ($\varepsilon$=2)    & \textbf{0.0296} & \textbf{0.0234}   & \textbf{0.0238}  & 19.79   & 16.40  & 17.41   & 24.4210 & 22.4675 & 21.7728 \\
\bottomrule
\end{tabular}
% \vspace{-0.45cm}
\end{table*}

\begin{figure*}[htb]
 \centering
   \begin{tabular}{@{}ccc@{}}
   \begin{minipage}{.3\textwidth}
    \includegraphics[width=\textwidth]{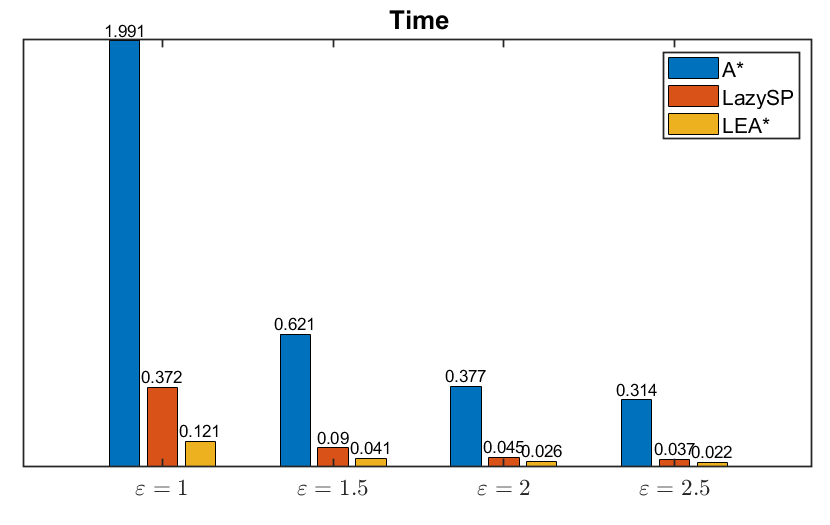}
    % \captionof*{figure}{(a)}
   \end{minipage} &
    \begin{minipage}{.3\textwidth}
    \includegraphics[width=\textwidth]{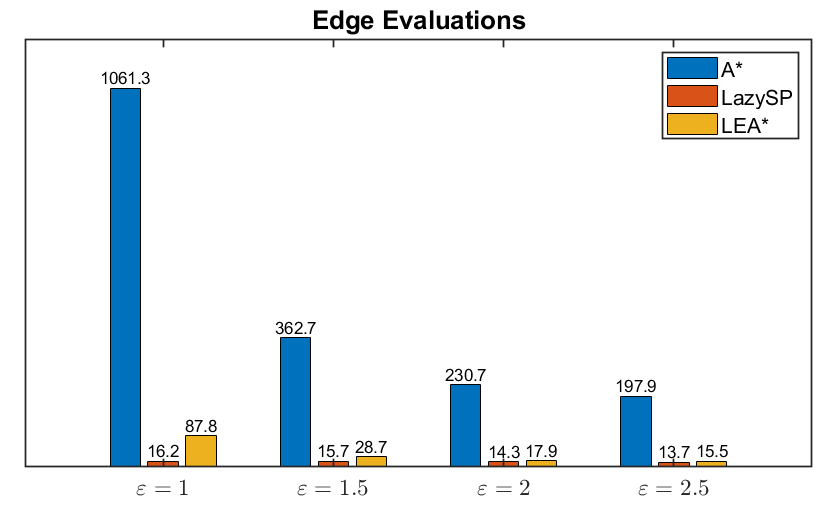}
    % \captionof*{figure}{(b)}
   \end{minipage} &
   \begin{minipage}{.3\textwidth}
    \includegraphics[width=\textwidth]{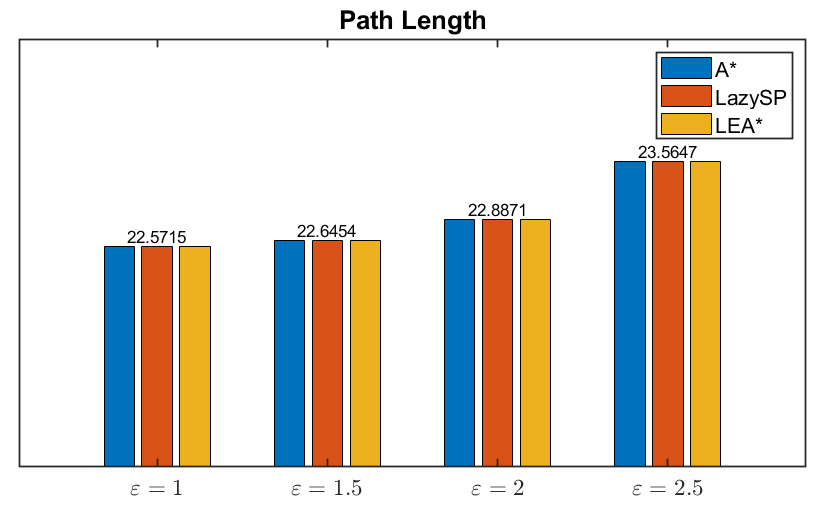}
    % \captionof*{figure}{(c)}
   \end{minipage} 
  \end{tabular}
  \caption{Planning results for a 7DOF manipulator.}
  \label{singleArm_compare}
\end{figure*}

\section{Simulation and Experiments}

In this section, we first provide simulation results for LEA* and CBLS. For single-manipulator planning, we compare three SAPF algorithms: LEA, A*, and LazySP.
For multi-manipulator planning, we consider scenarios with up to five manipulators operating in close proximity and compare the proposed CBLS algorithm with CBS and RRT-connect.
CBS uses A* for low-level SAPF planning and RRT-connect plans in the composite state space.
Then, real-world experiments based on the CBLS algorithm are conducted with two UR5 manipulators.

\begin{figure}[htb]
    \centering
    \begin{subfigure}[b]{0.45\columnwidth}
        \centering
        \includegraphics[width=1\columnwidth]{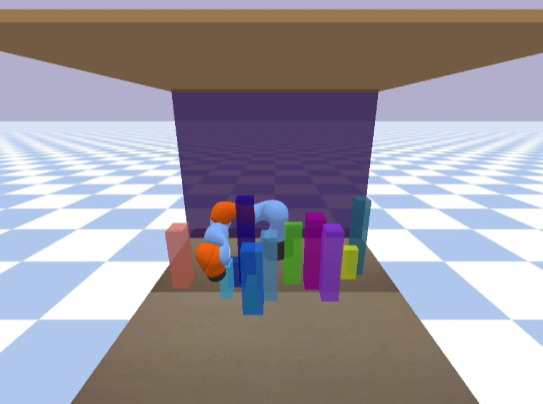}
        %\caption{}
    \end{subfigure}
    \begin{subfigure}[b]{0.45\columnwidth}
        \centering
        \includegraphics[width=1\columnwidth]{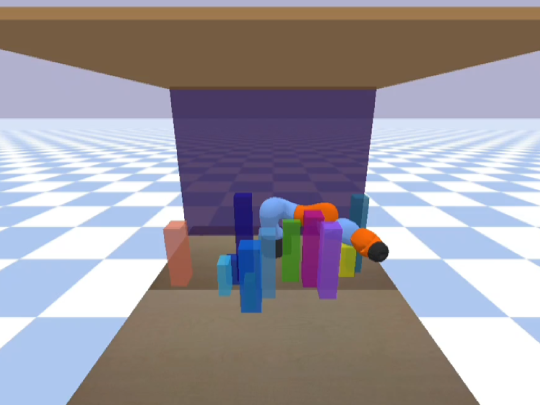}
        %\caption{}
    \end{subfigure}
    \caption{Single manipulator tabletop example.}
    \label{7Dexample}
\end{figure}

\subsection{LEA* Simulation Results}

We consider tabletop planning problems for a 7-DOF KUKA manipulator. 
The PyBullet simulator \cite{Coumans2021Pybullet} is adopted for the simulations.
The planning environment is shown in Figure~\ref{7Dexample}.
Cubic obstacles are randomly placed on the tabletop.
We generate different environments by varying the number, the location, and the size of obstacles in each environment.
The environment is divided into sparse, medium, and cluttered environments where the number of obstacles are $4$, $8$, and $12$, respectively.
A back wall and a ceiling are used to limit the operational space of the manipulator. 
The width, depth, and height of each obstacle are sampled uniformly and randomly from their respective intervals.
The locations of the obstacles on the table are also sampled.
We sampled 30 environments in total, and graphs with different sizes were precomputed.  
We consider small graphs, medium graphs, and large graphs with $N = 1{,}000$, $N = 5{,}000$, and $N = 10{,}000$, respectively. 
We use the Manhattan distance to sum the path lengths across individual degrees of freedom and across different manipulators. We set $\text{max\_count} = 2000$, $\text{m\_dist} = 3, 2.7, 2.5$, respectively, and the connecting radius $r = 5$, $4.5$, $4.2$, respectively.
By randomly sampling 50 start-goal queries for each environment and graph combination, the total number of planning problems is $4{,}500$.

The planning results for A*, LazySP, and LEA* are given in Table~\ref{tab:7Dtime}. The three algorithms were tested using the same $4{,}500$ problems. The statistics vary considerably across different planning problems, and Table~\ref{tab:7Dtime} reports the average results.
A*, LazySP, and LEA* use a cost-to-go heuristic $h(\cdot)$ to guide the graph search. 
While searching with an admissible heuristic guarantees path optimality, using an inflated heuristic has been shown to reduce planning time at the cost of optimality \cite{gammell2020batch, Cohen2014Single}.  
Bounded suboptimal solutions can be obtained by inflating $h(\cdot)$ with a factor $\varepsilon$ \cite{Pearl1984}. 
Results with inflation factor $\varepsilon = 2$ are also reported in Table~\ref{tab:7Dtime}.

For A* and LEA*, the planning time is approximately proportional to the number of evaluated edges during the search.
Using lazy edge evaluation, LEA* reduces edge evaluation significantly compared to A*.  
LazySP assumes all edges are collision-free and uses A* to find the best path.
Then, the best path is checked for collision. 
If a collision happens, the graph is updated, and a new planning problem is solved using A*.
LazySP is shown to evaluate the minimum number of edges (edge optimal) at the cost of repeatedly calling the A* algorithm in the inner loop.
As shown in Table~\ref{tab:7Dtime}, even if LazySP evaluates the minimum number of edges, LEA* is up to three times faster than LazySP for $N = 5,000$ and $N = 10,000$, as LazySP requires more graph operations such as priority queue updating and vertex expansions.
LEA* uses the least amount of time to find the same solution compared to the other two algorithms.
It performs 7DOF manipulator planning at 30-50 Hz.

In Table~\ref{tab:7Dtime}, the planning time and edge evaluations are further reduced with an inflation factor $\varepsilon = 2$. 
More results for different inflation factors are shown in Figure~\ref{singleArm_compare}.
With only a slight increase in path length, the planning time decreases considerably, and the edge evaluations of LEA* approach those of LazySP. 
A study regarding different graph sizes is provided in Figure~\ref{LEAstar_inflation2}.
LEA* with $\epsilon = 2$ is applied to the same set of planning problems with varying graph sizes $N$.
As the graph size increases, the success rate improves and the path length decreases.
Note that the planning time also decreases on average. 
This improvement is due to the inflation factor, which helps focus the search.
Also, a denser graph reduces detours between the start and goal. 
As shown in Table~\ref{tab:7Dtime}, the planning time may increase with larger $N$ when $\epsilon = 1$.

\begin{figure}[htb]
    \centering
    \begin{subfigure}[b]{0.46\columnwidth}
        \centering
        \includegraphics[width=1\columnwidth]{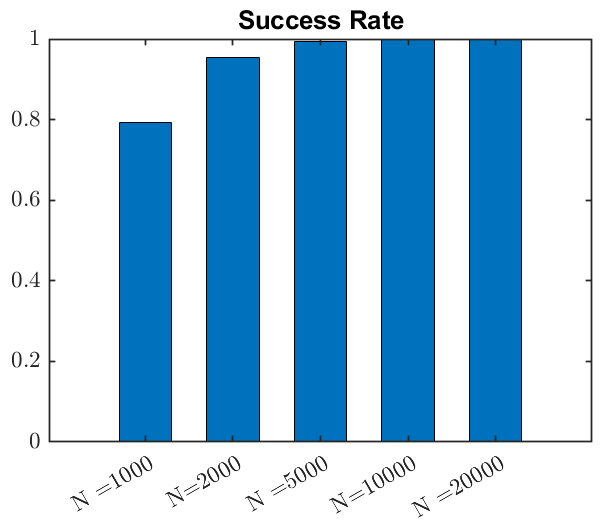}
        %\caption{}
    \end{subfigure}
    \begin{subfigure}[b]{0.46\columnwidth}
        \centering
        \includegraphics[width=1\columnwidth]{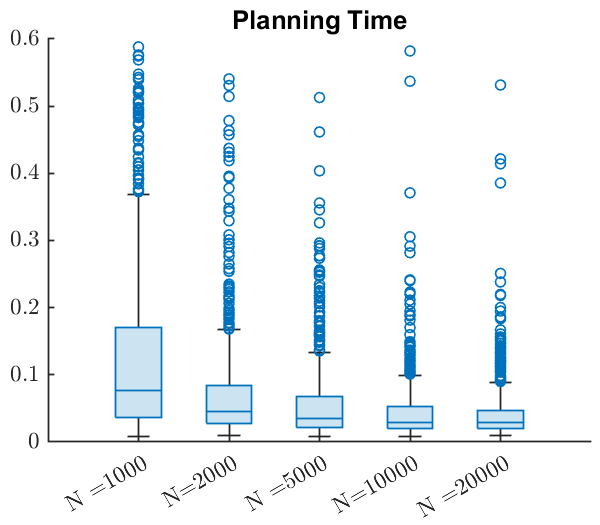}
        %\caption{}
    \end{subfigure}
    \begin{subfigure}[b]{0.45\columnwidth}
        \centering
        \includegraphics[width=1\columnwidth]{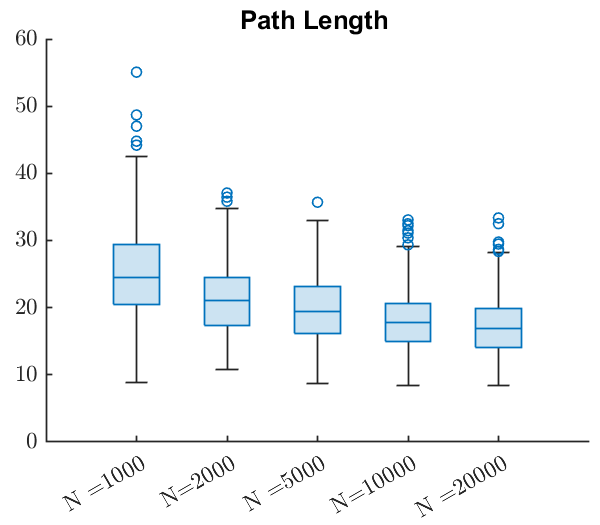}
        %\caption{}
    \end{subfigure}
    \begin{subfigure}[b]{0.45\columnwidth}
        \centering
        \includegraphics[width=1\columnwidth]{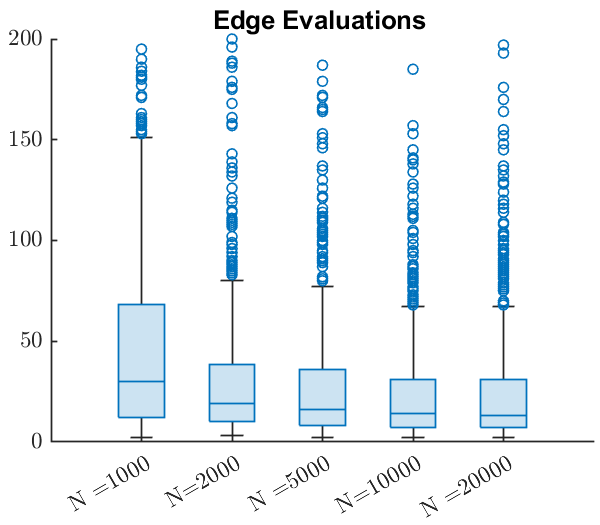}
        %\caption{}
    \end{subfigure}
    \caption{Results for different parameter $N$. The results are obtained by testing LEA* ($\epsilon$ = 2) on the same set of problems by varying $N$.}
    \label{LEAstar_inflation2}
\end{figure}

\begin{figure}[t]
    \centering
    \begin{subfigure}[b]{0.45\columnwidth}
        \centering
        \includegraphics[width=1\columnwidth]{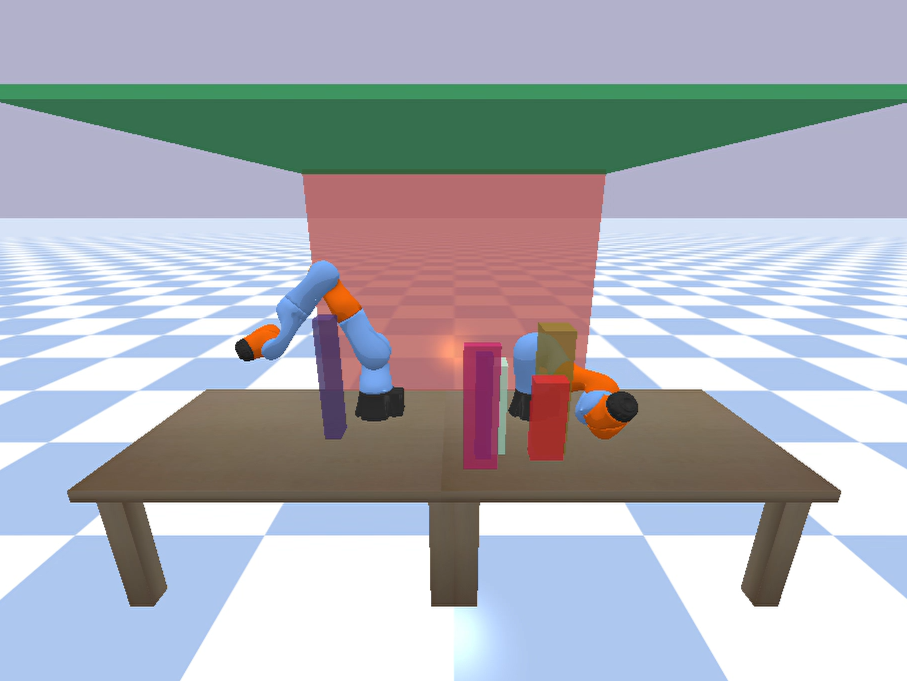}
        %\caption{Two manipulators}
    \end{subfigure}
    \begin{subfigure}[b]{0.45\columnwidth}
        \centering
        \includegraphics[width=1\columnwidth]{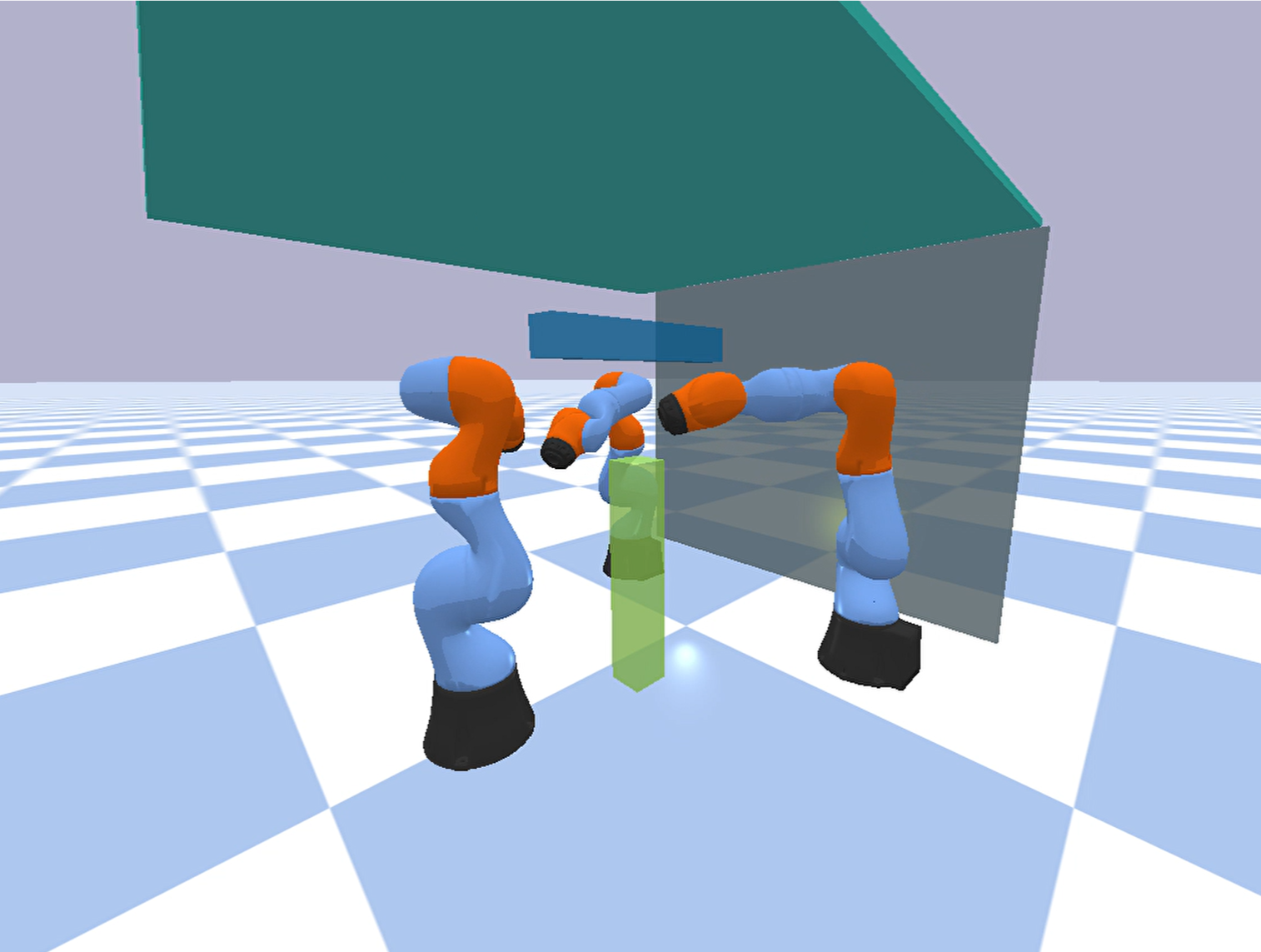}
        %\caption{Three manipulators}
    \end{subfigure}
    \begin{subfigure}[b]{0.45\columnwidth}
        \centering
        \includegraphics[width=1\columnwidth]{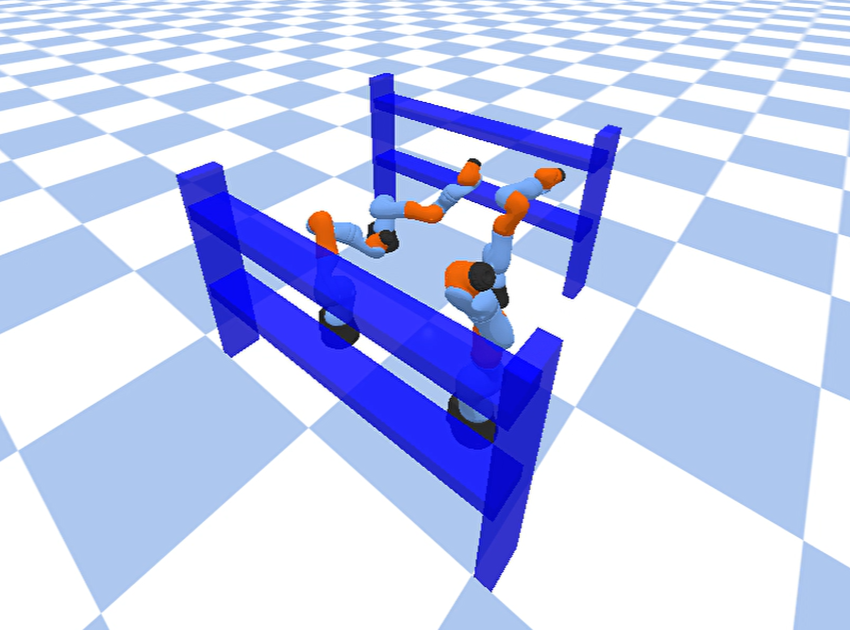}
        %\caption{Four manipulators}
    \end{subfigure}
    \begin{subfigure}[b]{0.45\columnwidth}
        \centering
        \includegraphics[width=1\columnwidth]{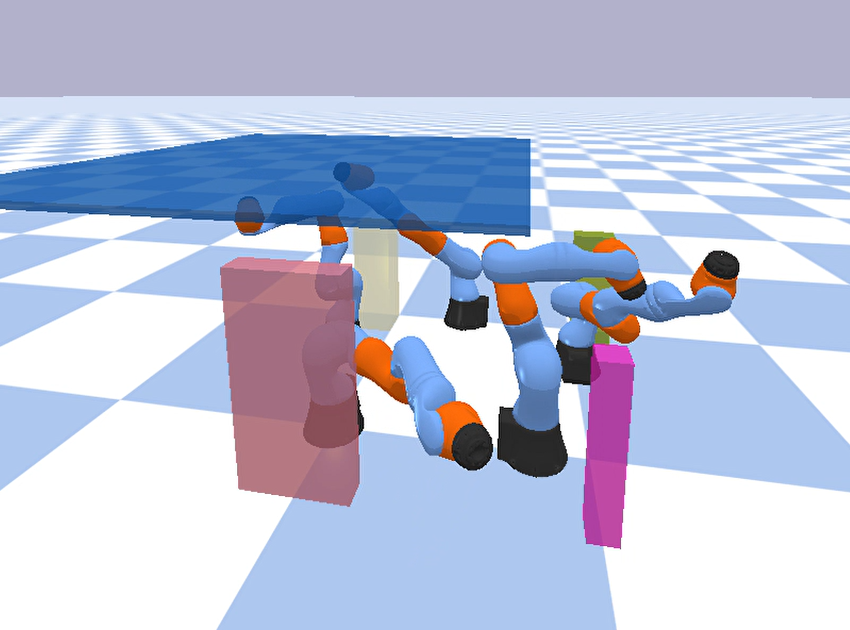}
        %\caption{Five manipulators}
    \end{subfigure}
    \caption{Multi-manipulator planning.}
    \label{multi-arm}
\end{figure}

\begin{figure}[htb]
    \centering
    \includegraphics[width=0.8\columnwidth]{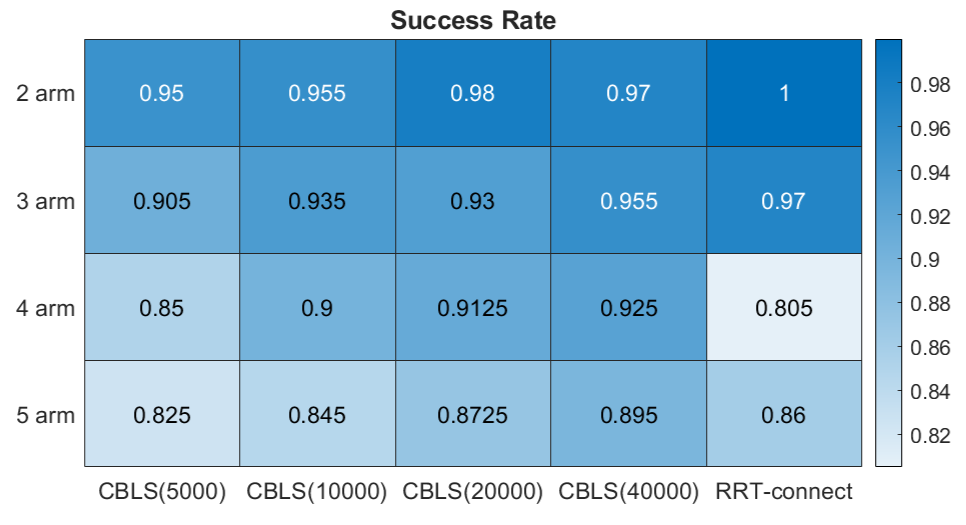}
    \caption{Multi-manipulator planning Success Rate (SR). SR decreases as the number of manipulators increases. SR increases as the graph size parameter $N$ increases.}
    \label{SR}
\end{figure}

\subsection{CBLS Simulation Results}

Planning problems for two-manipulator, three-manipulator, four-manipulator, and five-manipulator systems were studied. The planning environments are shown in Figure~\ref{multi-arm}.
Figure~\ref{multi-arm}(c) represents a shelf environment. In the other three environments, the obstacles are randomly sampled following the single manipulator case.
We construct  a graph with $N = 5{,}000$, $10{,}000$, $20{,}000$, and $40{,}000$ vertices, respectively. 
We set $\text{max\_count} = 2000$, $\text{m\_dist} = 2.7$, $2.5$, $2.2$, $2$, respectively,
and the connecting radius $r = 4.5$, $4.2$, $4.0$, $3.5$, respectively.
For each environment, we sampled 20 start-goal queries. 

CBLS, CBS, and RRT-connect were used to solve the same set of planning problems. 
Both CBLS and CBS used the same precomputed graphs (Section \ref{Sec:Graph}).
The only difference is that CBLS uses LEA* for SAPF while CBS uses A* for SAPF.
An additional time dimension is added to A* and LEA*.
Each vertex can move to its neighbors or stay at its current position at the next time step.

The RRT-connect algorithm is a representative algorithm for solving high-dimensional planning problems. It is probabilistically complete but does not guarantee properties such as asymptotic optimality. 
For many planning problems, it is one of the fastest algorithms for finding a feasible solution; however, the resulting solutions are typically of low quality \cite{Orthey2023Sampling}.
By combining the joint angle spaces of all manipulators, we use RRT-connect to solve planning problems with state spaces of up to 35 dimensions.
Since RRT-connect is a randomized algorithm, it produces different solutions for the same problem across runs.
We solve each problem ten times using RRT-connect and use the average results for comparison.
CBLS and CBS are deterministic algorithms.
For each planning problem, CBLS/CBS returns the same solution, with nearly identical planning times across runs. 

We set the maximum planning time to 15 seconds.
The planning query fails if the algorithm returns no solution or if the planning time exceeds 15 seconds. 
The success rate of multi-manipulator planning is given in Figure~\ref{SR}.
For CBLS, the planning success rate increases as the graph size $N$ increases.
The success rate decreases with the number of manipulators.
RRT-Connect achieves high success rates for two- and three-manipulator planning; however, its performance drops rapidly for four and five manipulators, as the planning time often exceeds the 15-second limit.

\begin{comment}
\begin{table*}
\caption{Planning results for two manipulators}\label{tab:two-arm}
\centering
\begin{tabular}{L{1.5cm}|L{0.89cm}L{0.93cm}L{0.93cm}L{0.93cm}|L{0.9cm}L{0.93cm}L{0.93cm}L{0.93cm}|L{0.9cm}L{0.93cm}L{0.93cm}L{0.93cm}}
\toprule
\hline
\multirow{4}{1.5cm} { } & \multicolumn{4}{c|}{Planning time (s)} & \multicolumn{4}{c|}{Path length} & \multicolumn{4}{c}{Edge evaluations} \\
\cline{2-13}  &  N=${5,000}$   & ${10,000}$   &  ${20,000}$   &  ${40,000}$   &  N=${5,000}$   & ${10,000}$   &  ${20,000}$   &  ${40,000}$  &  N=${5,000}$   & ${10,000}$   &  ${20,000}$   &  ${40,000}$     \\
\hline     
CBS ($\varepsilon$=1)   & 0.9719 & 1.2903   & 1.5009  & 1.6520  & 38.1141 & 33.5926 & 30.8849 & \textbf{30.2553}   & 4460.8  & 6057 & 7889 & 9722.8 \\
CBLS ($\varepsilon$=1)  & 0.2779 & 0.2654   & 0.2803  & 0.3426  & 38.1141 & 33.5926 & 30.8849 & \textbf{30.2553}   & 1281.3   & 1249  & 1414.3 & 1883  \\
CBS ($\varepsilon$=2)   & 0.2910 & 0.3595   & 0.3176  & 0.4289  & 38.8133  & 35.0762  & 32.4085  & 31.6018   & 1349.8   & 1604.7  & 1568.5 & 2281.8 \\
CBLS ($\varepsilon$=2)  & \textbf{0.0708} & \textbf{0.0716}  & \textbf{0.0507}  & \textbf{0.0757} & 38.8133 & 35.0762  & 32.4085 & 31.6018 & \textbf{316.1}  & \textbf{298.1}  & \textbf{195.7} & \textbf{346.4} \\
\cline{2-13}
RRT-connect & \multicolumn{4}{c|}{0.3579} & \multicolumn{4}{c|}{57.5607} & \multicolumn{4}{c}{75.5642} \\
\bottomrule
\end{tabular}
% \vspace{-0.45cm}
\end{table*}
\end{comment}

\begin{table*}
\caption{Planning results for three manipulators}\label{tab:three-arm}
\centering
\begin{tabular}{L{1.5cm}|L{0.89cm}L{0.93cm}L{0.93cm}L{0.93cm}|L{0.9cm}L{0.93cm}L{0.93cm}L{0.93cm}|L{0.9cm}L{0.93cm}L{0.93cm}L{0.93cm}}
\toprule
\hline
\multirow{4}{1.5cm} { } & \multicolumn{4}{c|}{Planning time (s)} & \multicolumn{4}{c|}{Path length} & \multicolumn{4}{c}{Edge evaluations} \\
\cline{2-13}  &  N=${5,000}$   & ${10,000}$   &  ${20,000}$   &  ${40,000}$   &  N=${5,000}$   & ${10,000}$   &  ${20,000}$   &  ${40,000}$  &  N=${5,000}$   & ${10,000}$   &  ${20,000}$   &  ${40,000}$     \\
\hline     
CBS ($\varepsilon$=1)   & 0.8770 & 1.3198  & 1.9067  & 2.1713  & 60.9560 & 51.6858 & 48.4792 & \textbf{46.1714}   & 3817.1  & 7063.5 & 10786 & 13214 \\
CBLS ($\varepsilon$=1)  & 0.2586 & 0.3313  & 0.4719  & 0.6278  & 60.9560 & 51.6858 & 48.4792 & \textbf{46.1714}   & 1266.2   & 1647.5  & 2162.3 & 2888.5  \\
CBS ($\varepsilon$=2)   & 0.2692 & 0.3967  & 0.4201  & 0.4863  & 62.2471 & 53.8270 & 50.9448 & 48.5605   & 1292.3   & 2071.1  & 2228.8 & 2660.8 \\
CBLS ($\varepsilon$=2)  & \textbf{0.0985} & \textbf{0.1107}  & \textbf{0.1321}  & \textbf{0.1982} & 62.2471 & 53.8270  & 50.9448 & 46.5605 & \textbf{383.65}  & \textbf{395.97}  & \textbf{342.09} & \textbf{432.5} \\
\cline{2-13}
RRT-connect & \multicolumn{4}{c|}{0.7631} & \multicolumn{4}{c|}{83.0737 } & \multicolumn{4}{c}{109.4276  } \\
\bottomrule
\end{tabular}
% \vspace{-0.45cm}
\end{table*}

\begin{comment}
\begin{table*}
\caption{Planning results for four manipulators}\label{tab:four-arm}
\centering
\begin{tabular}{L{1.5cm}|L{0.89cm}L{0.93cm}L{0.93cm}L{0.93cm}|L{0.9cm}L{0.93cm}L{0.93cm}L{0.93cm}|L{0.9cm}L{0.93cm}L{0.93cm}L{0.93cm}}
\toprule
\hline
\multirow{4}{1.5cm} { } & \multicolumn{4}{c|}{Planning time (s)} & \multicolumn{4}{c|}{Path length} & \multicolumn{4}{c}{Edge evaluations} \\
\cline{2-13}  &  N=${5,000}$   & ${10,000}$   &  ${20,000}$   &  ${40,000}$   &  N=${5,000}$   & ${10,000}$   &  ${20,000}$   &  ${40,000}$  &  N=${5,000}$   & ${10,000}$   &  ${20,000}$   &  ${40,000}$     \\
\hline     
CBS ($\varepsilon$=1)   & 1.1568 & 1.4060  & 1.6869  & 2.2472  & 69.9143 & 61.5470 & 58.0421 & \textbf{56.7661}   & 5133.1  & 6620.8 & 9262 & 12919 \\
CBLS ($\varepsilon$=1)  & 0.2777 & 0.2290  & 0.2637  & 0.4198  & 69.9143 & 61.5470 & 58.0421 & \textbf{56.7661}   & 974.986   & 717.53  & 831.17 & 1389.4  \\
CBS ($\varepsilon$=2)   & 0.2970 & 0.3222  & 0.4308  & 0.5158  & 72.2172 & 64.8581 & 61.6884 & 60.2417   & 1216.5   & 1395.6   & 1944.9 & 2356 \\
CBLS ($\varepsilon$=2)  & \textbf{0.1050} & \textbf{0.1102}  & \textbf{0.1392}  & \textbf{0.1915} & 72.2172 & 64.8581  & 61.6884 & 60.2417 & \textbf{194.88}  & \textbf{155.07}  & \textbf{177.06} & \textbf{169.85} \\
\cline{2-13}
RRT-connect & \multicolumn{4}{c|}{1.5563} & \multicolumn{4}{c|}{120.7914} & \multicolumn{4}{c}{229.2273} \\
\bottomrule
\end{tabular}
% \vspace{-0.45cm}
\end{table*}
\end{comment}

\begin{table*}[ht]
\caption{Planning results for five manipulators}\label{tab:five-arm}
\centering
\begin{tabular}{L{1.5cm}|L{0.89cm}L{0.93cm}L{0.93cm}L{0.93cm}|L{0.9cm}L{0.93cm}L{0.93cm}L{0.93cm}|L{0.9cm}L{0.93cm}L{0.93cm}L{0.93cm}}
\toprule
\hline
\multirow{4}{1.5cm} { } & \multicolumn{4}{c|}{Planning time (s)} & \multicolumn{4}{c|}{Path length} & \multicolumn{4}{c}{Edge evaluations} \\
\cline{2-13}  &  N=${5,000}$   & ${10,000}$   &  ${20,000}$   &  ${40,000}$   &  N=${5,000}$   & ${10,000}$   &  ${20,000}$   &  ${40,000}$  &  N=${5,000}$   & ${10,000}$   &  ${20,000}$   &  ${40,000}$     \\
\hline     
CBS ($\varepsilon$=1)   & 0.3895 & 0.2646  & 0.4970  & 0.4327  & 85.7357 & 76.3592 & 71.5091 & \textbf{70.048}   & 2096.6  & 1411.9 & 2709.3 & 2778.6 \\
CBLS ($\varepsilon$=1)  & 0.2417 & 0.2113  & 0.2971  & 0.3918  & 85.7357 & 76.3592 & 71.5091 & \textbf{70.048}   & 1125.7   & 823.52  & 1103.5 & 1536.1  \\
CBS ($\varepsilon$=2)   & 0.1025 & 0.1206  & 0.1486  & 0.2647  & 88.9099 & 81.1278 & 76.0302 & 74.7359   & 334.35   & 441.09  & 601.39 & 1254.2 \\
CBLS ($\varepsilon$=2)  & \textbf{0.1231} & \textbf{0.1278}  & \textbf{0.1467}  & \textbf{0.2208} & 88.9099 & 81.1278  & 76.0302 & 74.7359 & \textbf{171.92}  & \textbf{143.99}  & \textbf{141.74} & \textbf{193.29} \\
\cline{2-13}
RRT-connect & \multicolumn{4}{c|}{2.5072} & \multicolumn{4}{c|}{174.0136} & \multicolumn{4}{c}{280.6029} \\
\bottomrule
\end{tabular}
% \vspace{-0.45cm}
\end{table*}

Multi-manipulator planning results are given in Tables~\ref{tab:three-arm}-\ref{tab:five-arm}.
For the planning results of two manipulators and four manipulators, please refer to the tables shown in the supplementary video.
The planning time of CBLS outperforms CBS and RRT-connect for all cases, on average. 
Compared to RRT-connect, which directly composes the DOF of all manipulators, our CBLS combines precomputation, fast SAPF, and conflict-based search.
CBLS and CBS also find paths with shorter lengths compared to RRT-connect.
In CBLS and CBS, edge evaluation requires checking the collision of an edge for a single manipulator,
while in RRT-connect, edge evaluation requires checking the collision of all manipulators.
From the results of Tables~\ref{tab:three-arm}-\ref{tab:five-arm}, it is evident that using an inflated heuristic is beneficial.
It greatly decreases the planning time, returning bounded suboptimal solutions that only increase the (resolution) optimal path length by a small factor.

Dynamic obstacles can be handled by replanning. 
At every time step, the predicted trajectories of the moving obstacles within a future time horizon are used for planning. Beyond this horizon, the dynamic obstacles are not considered. 
Thus, the planned multi-manipulator path is guaranteed to be collision-free within the upcoming horizon. 
As the manipulators move along the planned path, the obstacle trajectories for the next horizon become available. 
Replanning is triggered if any potential collision is detected. 
This procedure is repeated until the manipulators reach their goals. 
Path planning for three manipulators, while avoiding dynamic obstacles, is demonstrated in Figure~\ref{dynamicObstacle}.
The planning time is correlated with the duration of the predicted obstacle trajectories.
In this example, the planning frequency is above 5 Hz, and the time horizon can be set to 0.2 s.
With a faster planner, less information about the moving obstacles is required.

\begin{figure}[htb]
    \centering
    \begin{subfigure}[b]{0.45\columnwidth}
        \centering
        \includegraphics[width=1\columnwidth]{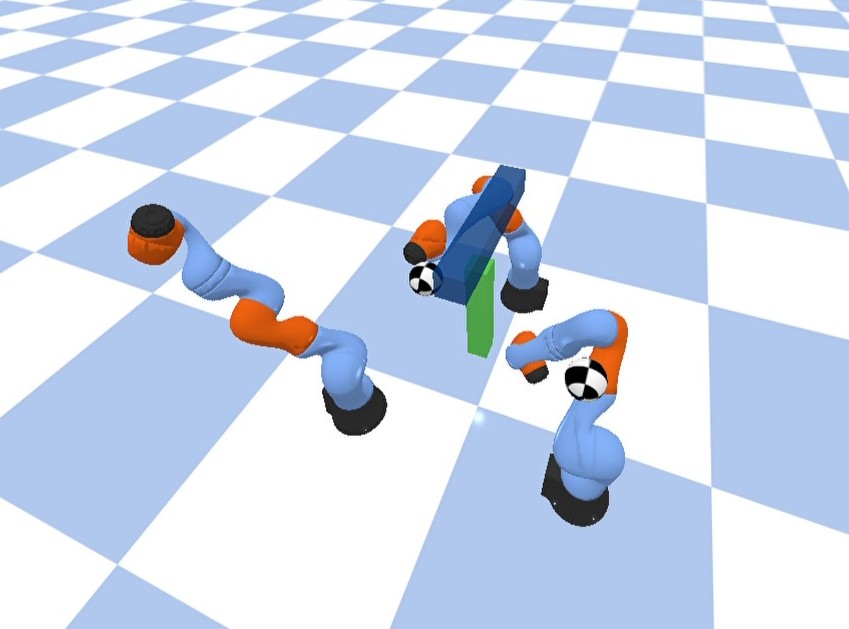}
    \end{subfigure}
    \begin{subfigure}[b]{0.45\columnwidth}
        \centering
        \includegraphics[width=1\columnwidth]{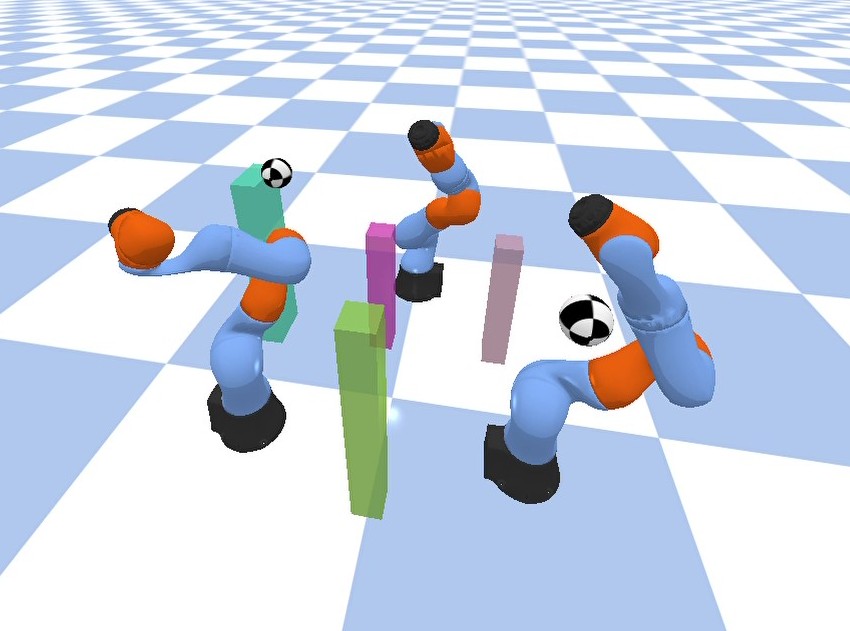}
    \end{subfigure}
    \caption{Planning with dynamic obstacles.}
    \label{dynamicObstacle}
\end{figure}

\begin{figure}[htb]
    \centering
    \begin{subfigure}[b]{0.48\columnwidth}
        \centering
        \includegraphics[width=1\columnwidth]{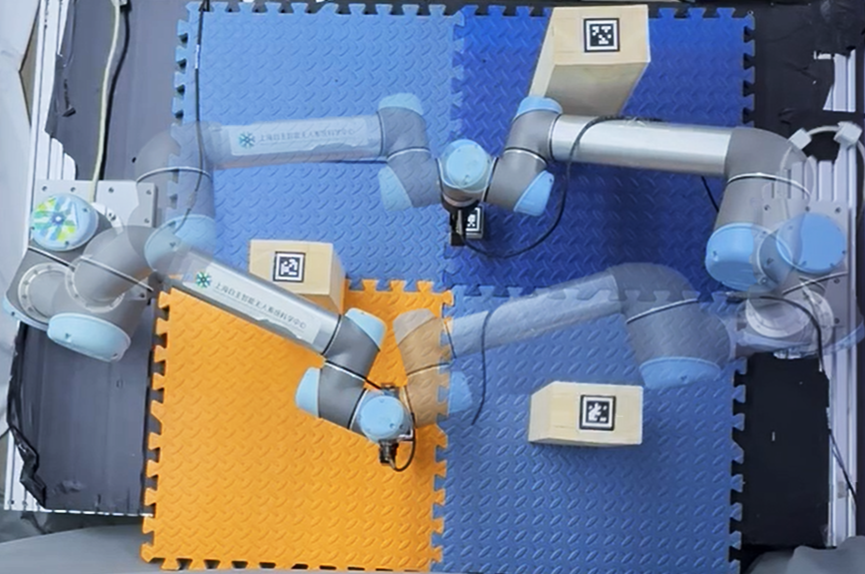}
    \end{subfigure}
    \begin{subfigure}[b]{0.48\columnwidth}
        \centering
        \includegraphics[width=1\columnwidth]{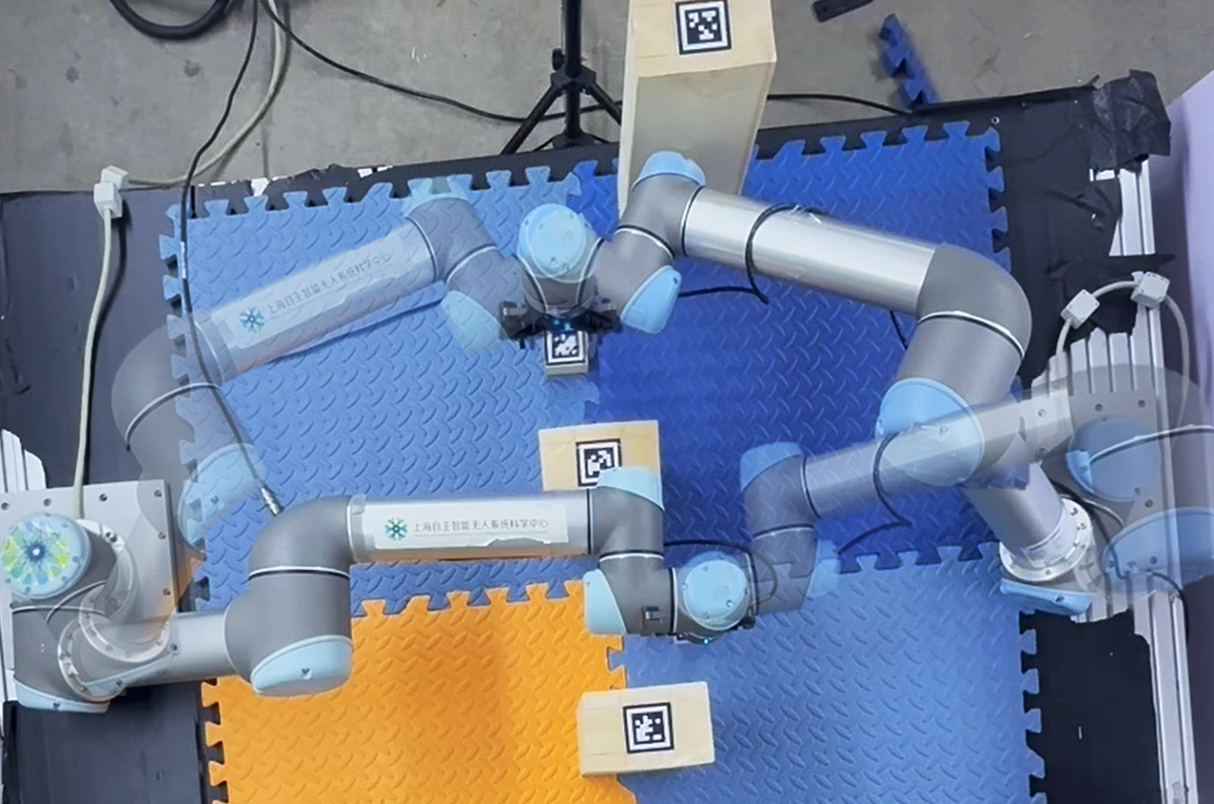}
    \end{subfigure}
    \caption{Two planning examples. CBLS plans collision-free paths for the UR5 manipulators to move from their starting configurations to the goals (shown in transparent mode).}
    \label{Experiment_Results}
\end{figure}

\subsection{CBLS Experimental Results}

We performed extensive real-world experiments using the CBLS algorithm for two-manipulator planning. 
The environment setup is shown in Figure \ref{Experiment_Setting} and \ref{Experiment_Results}.
Two UR5 manipulators were used, and obstacles (shown as wooden bricks) were randomly placed. 
Similar to the KUKA manipulator example, a simulation environment was built to replicate the experiment environment.
We use an external camera and QR codes to identify and localize obstacles.
After calibrating the camera and the manipulator system, we compute the obstacle locations in real time in the manipulator’s coordinate frame.
This approach provides a relatively accurate model of the environment.
Future work includes conducting more complex experiments, adopting more sophisticated perception methods, and planning under uncertainty.
Our experiments show that CBLS achieves real-time planning.
Manipulator self-collision, manipulator-obstacle collision, and manipulator-manipulator collision were avoided successfully during the experiments.
Snapshots of the two manipulators switching goals are shown in Figure~\ref{Experiment_Results}.
Animations of the experiments and simulations can be found at \texttt{\url{https://youtu.be/qC--BcjcdwA}}.
% \texttt{\url{https://cbls-multi-arm.netlify.app}}.

\section{Conclusion}

This paper introduces the CBLS algorithm for fast planning in high-dimensional spaces.
CBLS is based on CBS, a start-of-the-art backbone for many MAPF algorithms and is tailored for manipulator planning.
It introduces two enhancements to CBS for solving SAPF problems more efficiently.
First, a lazily evaluated graph is used with controlled sparsity, which is precomputed for online planning.
This graph is environment-independent and can be applied to any number of manipulators.
Second, Lazy Edged-based A* (LEA*) is utilized as the low-level planner 
of CBS for efficient SAPF.
It is shown that LEA* is complete, finds the optimal solution, is optimally vertex efficient, and reduces the number of collision checks using lazy search.
Various simulations are conducted to demonstrate the theory. 
The advantages of LEA* and CBLS are illustrated by comparing them with their respective competing algorithms. 
Practical applications of CBLS on a two-manipulator system are also conducted to validate the proposed methods.

% \balance

% \newpage

 


\begin{IEEEbiography}
[{\includegraphics[width=1in,height=1.25in,clip,keepaspectratio]{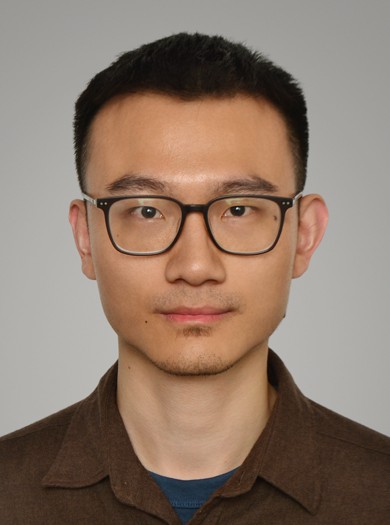}}]
{Dongliang Zheng}(Member, IEEE) received his Ph.D. degree in aerospace engineering from Georgia Institute of Technology, Atlanta, GA, USA, in 2024, the M.S. degree in control engineering from Shanghai Jiao Tong University, Shanghai, China, in 2018, and the B.Eng. degree in automation from Northeastern University, Shenyang, China, in 2015. He is currently an Associate Professor at the Shanghai Research Institute for Intelligent Autonomous Systems, Tongji University.
His research interests include robot motion planning and autonomy.
\end{IEEEbiography}

\begin{IEEEbiography}
[{\includegraphics[width=1in,height=1.25in,clip,keepaspectratio]{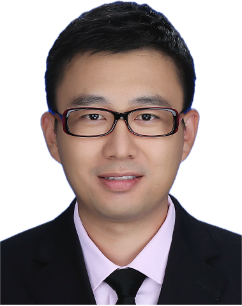}}]{Zhipeng Wang}(Member, IEEE) received the M.S. degree in mechanical manufacturing and automation from Zhejiang University, Hangzhou, China, in 2011, and the Ph.D. degree in control science and engineering from Tongji University, Shanghai, China, in 2015. He is currently an Associate Professor with the Department of Control Science and Engineering, Tongji University. 

His research interests include robot motion planning, mechatronics, and dynamics.
\end{IEEEbiography}

\begin{IEEEbiography}
[{\includegraphics[width=1in,height=1.25in,clip,keepaspectratio]{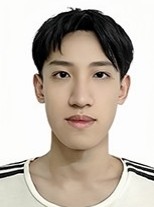}}]{Siqi Wang} received the B.E. degree in  Control Science and Engineering from Tongji University, Shanghai, China, in 2023. He is currently pursuing the Ph.D. degree with the Shanghai Research Institute for Intelligent Autonomous Systems, Tongji University, China. 

His research interests include the motion planning and safe collaborative operation of dual-arm robots.
\end{IEEEbiography}

\begin{IEEEbiography}
[{\includegraphics[width=1in,height=1.25in,clip,keepaspectratio]{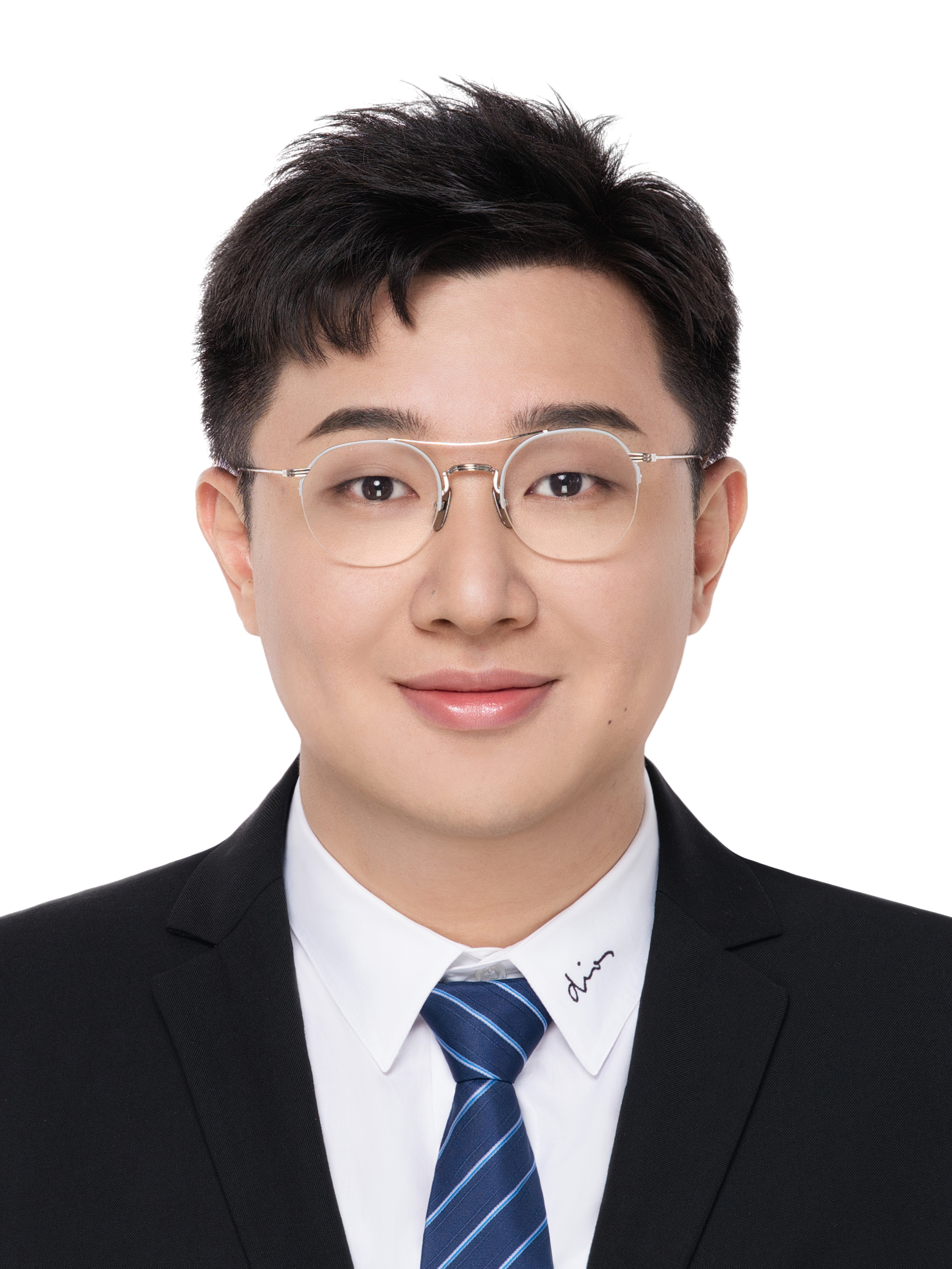}}]{Yuxi Lu}(Member, IEEE) received the Ph.D. degree in Medical Engineering from the Graduate School of Science and Engineering, Chiba University, Chiba, Japan, in 2024, and the M.S. and B.S. degrees in Engineering from Tokyo University of Science, Tokyo, Japan, in 2021 and 2019, respectively.
He is currently an Assistant Professor with the Shanghai Research Institute for Intelligent Autonomous Systems, Tongji University, Shanghai, China. 

His current research interests focus on medical robotics and soft robotic systems. % particularly their applications in minimally invasive surgery and rehabilitation.
\end{IEEEbiography}

\begin{IEEEbiography}
[{\includegraphics[width=1in,height=1.25in,clip,keepaspectratio]{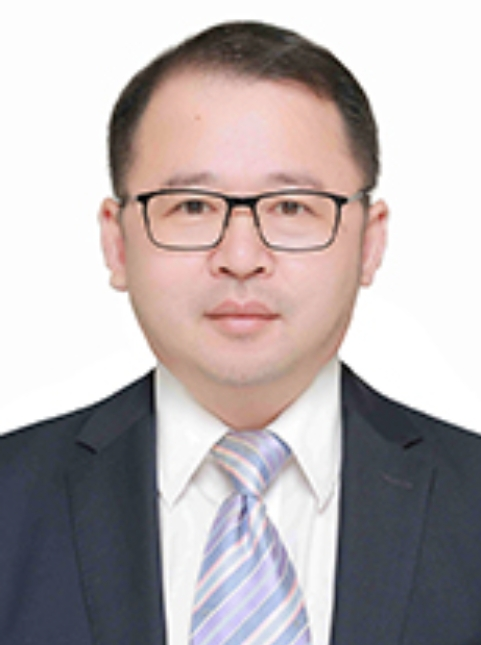}}]{Bin He}
(Senior Member, IEEE) received the Ph.D. degree in mechanical and electronic control engineering from Zhejiang University, Hangzhou, China, in 2001. From 2001 and 2003, he held Postdoctoral research appointments with the State Key Lab of Fluid Power Transmission and Control, Zhejiang University. He is currently a Professor with
the College of Electronics and Information Engineering, Tongji University, Shanghai, China.
His current research interests include intelligent robot control, biomimetic microrobots, and wireless networks.
\end{IEEEbiography}

\begin{IEEEbiography}
[{\includegraphics[width=1in,height=1.25in,clip,keepaspectratio]{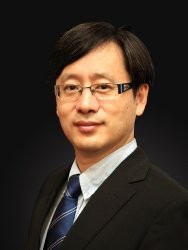}}]{Hesheng Wang}
(Senior Member, IEEE) received the B.Eng. degree in electrical engineering from the Harbin Institute of Technology, Harbin, China, in 2002, and the M.Phil. and Ph.D. degrees in automation and computer-aided engineering from The Chinese University of Hong Kong, Hong Kong, in 2004 and 2007, respectively.
He is currently a Distinguished Professor with the Department of Automation, Shanghai Jiao Tong University, Shanghai, China. His research interests include visual servoing, intelligent robotics, computer vision, and autonomous driving. 

Dr. Wang is an Associate Editor for Robotic Intelligence and Automation and the International Journal of Humanoid Robotics, a Senior Editor for IEEE/ASME Transactions on Mechatronics, an Editor-in-Chief for
Robot Learning. He was an Associate Editor for IEEE Transactions on Robotics from 2015 to 2019, an Associate Editor for IEEE Transactions on Automation Science and Engineering from 2021 to 2023, and an Editor for the Conference Editorial Board of the IEEE Robotics and Automation Society from 2022 to 2024. He was the General Chair of IEEE ROBIO 2022 and IEEE RCAR 2016, and Program Chair of the IEEE ROBIO 2014 and IEEE/ASME AIM 2019. He will be the General Chair of IEEE/RSJ IROS 2025.
\end{IEEEbiography}

\begin{IEEEbiography}
[{\includegraphics[width=1in,height=1.25in,clip,keepaspectratio]{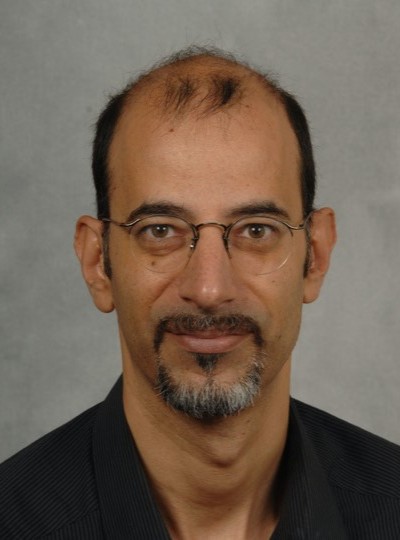}}]{Panagiotis Tsiotras}
(F’19) is the David and Andrew Lewis Chair Professor in the Daniel Guggenheim School of Aerospace Engineering at the Georgia Institute of Technology, Atlanta, GA, USA. He received his Ph.D. degree in aerospace engineering and an M.S. degree in mathematics from Purdue University, IN, USA, an M.S. degree in aerospace
engineering from Virginia Tech, VA, USA and an Eng. Dipl. in mechanical engineering from NTUA,
Athens, Greece. He has held visiting research appointments at MIT, JPL, INRIA Rocquencourt, and
Mines ParisTech.

His research interests include optimal control of nonlinear systems and ground, aerial, and space vehicle autonomy. He has served on the Editorial Boards of the Transactions on Automatic Control, the IEEE Control Systems
Magazine, the AIAA Journal of Guidance, Control and Dynamics, the Dynamic Games and Applications, and Dynamics and Control. He is the recipient of the NSF CAREER award, the Outstanding Aerospace Engineer
award from Purdue, and the Technical Excellence Award in Aerospace Control from IEEE. He is a Fellow of IEEE, AIAA and AAS.
\end{IEEEbiography}

\vfill

\end{document}